\documentclass[journal,twoside,web]{ieeecolor}
\usepackage{generic}
\usepackage{amsmath,amssymb,amsfonts}
\usepackage{graphicx}
\usepackage{textcomp}
\usepackage{hyperref}

\newcommand{\mE}{\ensuremath{\mathbb{E}}}

\newcommand{\mP}{\ensuremath{\mathbb{P}}}

\newcommand\gC{{\mathcal{C}}}

\newcommand\gE{{\mathcal{E}}}

\newcommand\gL{{\mathcal{L}}}
\newcommand\gM{{\mathcal{M}}}

\newcommand\gX{{\mathcal{X}}}
\newcommand\gY{{\mathcal{Y}}}

\providecommand{\keywords}[1]{\vspace*{0.5\baselineskip}\par\noindent\textbf{Keywords:} #1}

{%
\begin{list}{#1}{
\vspace{-\topsep}
\vspace{-\partopsep}
\setlength{\itemindent}{0cm}
\setlength{\rightmargin}{0cm}
\setlength{\listparindent}{0cm}
\settowidth{\labelwidth}{#1}
\setlength{\leftmargin}{\labelwidth}
\addtolength{\leftmargin}{\labelsep}
\setlength{\itemsep}{0cm}
}%
}%
{%
\end{list}
\vspace{-\topsep}
\vspace{-\partopsep}
}

{\begin{enumerate}%
\renewcommand{\theenumi}{\roman{enumi}}%
}%
{\end{enumerate}}

\usepackage{amsthm}

\theoremstyle{plain}
\newtheorem{theorem}{Theorem}
\newtheorem{corollary}{Corollary}

\newtheorem{lemma}{Lemma}
\newtheorem{assumption}{Assumption}

\theoremstyle{definition}
\newtheorem{definition}{Definition}
\newtheorem{condition}{Condition}
\newtheorem{example}{Example}
\theoremstyle{plain}

\theoremstyle{remark}
\newtheorem{remark}{Remark}

\usepackage{bbold}

\usepackage{amsmath,amssymb,amsthm, amsfonts, mathrsfs}

\usepackage{breqn}
\usepackage{graphicx}
\usepackage{xcolor}
\usepackage{comment}
\usepackage{dsfont}

\usepackage{cleveref}
\Crefname{equation}{}{}
\crefname{equation}{}{}
\usepackage{footmisc}
\usepackage{csquotes}
\let\labelindent\relax
\usepackage{enumitem}
\usepackage{bbm}

\usepackage{algorithm} 
\usepackage{algpseudocode}

\usepackage{soul} 

\usepackage{caption, subcaption}

\def\BibTeX{{\rm B\kern-.05em{\sc i\kern-.025em b}\kern-.08em
    T\kern-.1667em\lower.7ex\hbox{E}\kern-.125emX}}

\newcommand{\ie}{\textit{i.e., }}
\newcommand{\eg}{\textit{e.g., }}
    
\newcommand{\R}{\mathbb R}

\newcommand{\exit}{\Gamma}
\newcommand{\minf}{\Phi}
\newcommand{\entrance}{\Psi}
\newcommand{\maxf}{\Theta}

\newcommand{\Hess}[1]{\operatorname{Hess}{#1}}
\newcommand{\grad}[1]{\nabla{#1}}
\newcommand{\tr}{\operatorname{tr}}

\newcommand{\FwdInv}{F} 
\newcommand{\FwdInvExit}{G}
\newcommand{\FwdConv}{Q}
\newcommand{\FwdConvExit}{N}

\newcommand{\lemmaF}{U} 
\newcommand{\lemmaT}{H} 

\newcommand{\ave}{\rho}
\newcommand{\var}{\zeta}

\usepackage[sort, nocompress]{cite}
\usepackage{graphicx}
\usepackage[font=normalsize]{caption}
\usepackage{subcaption}
\usepackage{xcolor}
\usepackage{bbold}
\usepackage{cleveref}
\usepackage{footmisc}
\usepackage{csquotes}
\usepackage{dsfont}

\newcommand{\pr}{\mathbb{P}}

\newcommand{\D}{D}

\newcommand{\rev}[1]{#1}            

\begin{document}

\title{Generalizable Physics-Informed Learning for Stochastic Safety-Critical Systems}
\author{Zhuoyuan Wang, Albert Chern, Yorie Nakahira
\thanks{Zhuoyuan Wang and Yorie Nakahira are with the Department of Electrical and Computer Engineering, Carnegie Mellon University, PA 15213 USA  (e-mail: zhuoyuaw, ynakahir@andrew.cmu.edu).}
\thanks{Albert Chern is with the Department of Computer Science and Engineering, University of California San Diego, CA 92093 USA (e-mail: alchern@ucsd.edu).}}

\maketitle



\setcounter{page}{1}

\begin{abstract}

Accurate estimation of long-term risk is essential for the design and analysis of stochastic dynamical systems. Existing risk quantification methods typically rely on extensive datasets involving risk events observed over extended time horizons, which can be prohibitively expensive to acquire. Motivated by this gap, we propose an efficient method for learning long-term risk probabilities using short-term samples with limited occurrence of risk events. Specifically, we establish that four distinct classes of long-term risk probabilities are characterized by specific partial differential equations (PDEs). Using this characterization, we introduce a physics-informed learning framework that combines empirical data with physics information to infer risk probabilities. We then analyze the theoretical properties of this framework in terms of generalization and convergence. Through numerical experiments, we demonstrate that our framework not only generalizes effectively beyond the sampled states and time horizons but also offers additional benefits such as improved sample efficiency, rapid online inference capabilities under changing system dynamics, and stable computation of probability gradients. These results highlight how embedding PDE constraints—which contain explicit gradient terms and inform how risk probabilities depend on state, time horizon, and system parameters—improves interpolation and generalization between/beyond the available data.

\end{abstract}

\section{Introduction}

\rev{

Accurate estimation of long-term risk is critical for the design and analysis of stochastic dynamical systems. 
Miscalculations of risk can lead to safety and performance issues over extended periods, and short-term risk is not sufficient to capture when a safe state inevitably reaches unsafe regions in the future. 
Existing risk quantification methods often require sufficient data on risk events and long-term trajectories. However, obtaining such data (\eg through real-world operations) can be prohibitively costly. Moreover, when the risk of interest pertains to rare events, the volume of data needed by purely data-driven approaches increases disproportionately, making reliable estimation a significant challenge~\cite{mcneil2015quantitative}.

Motivated by these challenges, we propose Physics-Informed Probability Estimator (PIPE), a physics-informed learning framework to estimate risk probabilities in an efficient and generalizable manner. Fig.~\ref{fig:overview diagram} shows the overview diagram of the proposed PIPE framework. 
The framework consists of two key parts, namely the physics model part and the data-driven learning part. For the former part, we derive that the evolution of risk probabilities given stochastic system dynamics is governed by deterministic partial differential equations (PDEs) (Theorem~\ref{thm:InvariantProbability_MainTheorem1}-\ref{thm:ConvergenceProbability_MainTheorem4}).
For the latter part, we integrate the physics model knowledge from the PDEs with sampled data to train a physics-informed neural network.
We theoretically show that the proposed framework can generate risk probability estimates with bounded error, even when the data used for training only covers partial state space with short time horizons (Theorem~\ref{thm:full_pde_constraint}). We also show convergences of the proposed method to ground truth risk probabilities under mild assumptions (Theorem~\ref{thm:pinn_convergence}). In addition, we experimentally show that the proposed framework is able to accurately predict risk probabilities and their gradients on the \textit{entire} state-temporal domain, with only \textit{sparse} data sample in the \textit{sub-region} (Fig.~\ref{fig:generalization} and Fig.~\ref{fig:gradient}), enabling huge sample complexity and online computation time reduction (Fig.~\ref{fig:PINN MC log plot} and Table~\ref{tab:method_timing}). Furthermore, we show that the proposed framework can serve as a surrogate model to learn risk probabilities of varying parameterized system dynamics, with data only on a few dynamics in the distribution (Fig.~\ref{fig:varying parameter}).
}

\begin{figure*}
    \centering
    \includegraphics[width=0.99\textwidth]{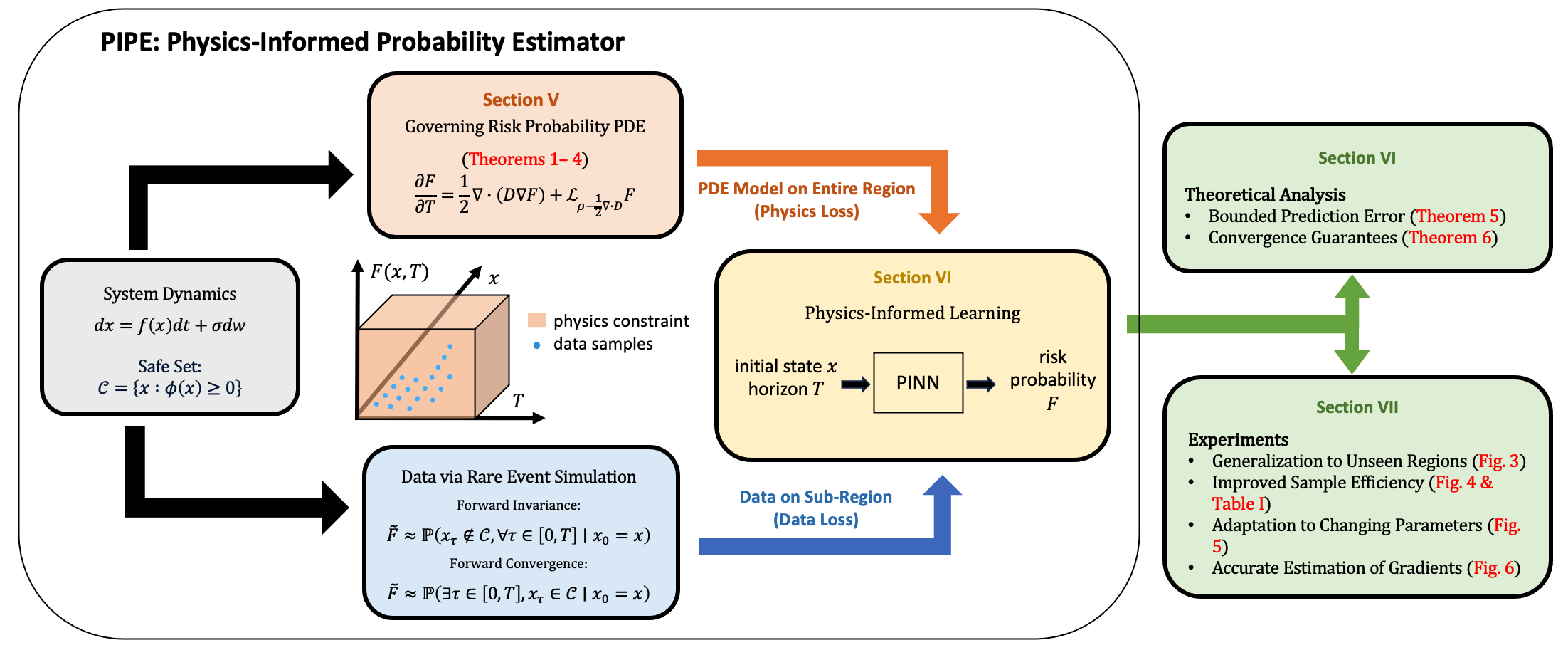}
    \caption{The overview diagram of the proposed PIPE framework. 
    The system takes form~\eqref{eq:x_trajectory} with safe set defined as~\eqref{eq:safe_region}. For the physics model, we derive that the mapping between state-time pair and the risk probability satisfies a governing convection diffusion equation (Theorem~\ref{thm:InvariantProbability_MainTheorem1}-\ref{thm:ConvergenceProbability_MainTheorem4}). For training data, one can acquire the empirical risk probabilities with any existing rare event simulation method, without the need to cover the entire time-state space. The PIPE framework uses physics-informed neural networks to learn the risk probability by fitting the empirical training data and using the physics model as constraints. PIPE gives more accurate and sample efficient risk probability and gradient predictions than Monte Carlo or its variants, and can generalize to unseen regions in the state space and unknown parameters in the system dynamics thanks to its integration of data and physics models. 
    }
    \label{fig:overview diagram}
\end{figure*}

The rest of the paper is organized as follows. In section~\ref{sec:related_work}, we list related work. In section~\ref{sec:preliminaries} we give preliminaries. In section~\ref{sec:problem_statement}, we state the problem formulation. In section~\ref{sec:safety_thm}, we derive the mathematical characterization of risk probability with the governing partial differential equations. In section~\ref{sec:PINN}, we present the physics-informed learning method for probability estimation and the corresponding theoretical analysis. In section~\ref{sec:experiments}, we show experiment results. Finally we conclude the paper in section~\ref{sec:conclusion}.

\section{Related Work}
\label{sec:related_work}

\subsection{Risk Probability Estimation}

\rev{
A wide range of techniques has been developed to estimate risk probabilities. These approaches vary in data availability, theoretical assumptions, and computational complexity. One line of work uses results from probability theory and dynamical systems. For example, probabilistic inequalities and martingale methods are used to bound the probability of deviating into an unsafe region~\cite{prajna2007framework,yaghoubi2020risk,santoyo2021barrier,cheng2020safe}. 
Large deviations theory provides asymptotic estimates of rare event probabilities by characterizing the decay rate of such probabilities~\cite{tong2022optimization}. 
While this line of work seeks to bound or approximate risk probabilities analytically, we characterize the exact risk probability (Theorem~\ref{thm:InvariantProbability_MainTheorem1}-\ref{thm:ConvergenceProbability_MainTheorem4}).

Another line of work uses Monte Carlo techniques to estimate risk or rare event probabilities~\cite{rubino2009rare}. 
Importance sampling methods are commonly used when sampling distributions are biased towards the failure events to efficiently obtain information about risk-sensitive states~\cite{janssen2013monte,cerou2012sequential}. The cross-entropy method iteratively optimizes the sampling distribution via cross-entropy minimization~\cite{uribe2021cross}.  Reinforcement learning has also been employed to guide the sampling process~\cite{yang2024rare}. Bayesian inference is used to steer samples toward near-optimal sampling distribution~\cite{ehre2023stein}. 
Subset simulation breaks the rare event into a sequence of more likely intermediate events, estimating the small failure probability as a product of conditional probabilities~\cite{zuev2015subset}. 
These data-driven approaches often selectively obtain risk event data for improved sample efficiency, and require data to be obtained from all states and time horizons of interest~\cite {rubino2009rare,janssen2013monte,zuev2015subset}. In contrast, our work addresses a different challenge of learning a surrogate model in a way that integrates data and physics information, and the proposed framework can also leverage these sampling techniques. Due to the use of physics information, the proposed method can infer long-term risk information from short-term safe event data (Theorem~\ref{thm:full_pde_constraint} and Fig.~\ref{fig:generalization}). This feature also contributes to improved sampled efficiency, particularly when the risk events are rare (Fig.~\ref{fig:PINN MC log plot}).
}


\subsection{Physics-informed Neural Networks}

\rev{
Physics-informed neural networks (PINNs) have emerged as a robust framework for integrating physical constraints—typically expressed as partial differential equations (PDEs)—directly into a neural network training~\cite{raissi2019physics,han2018solving,cuomo2022scientific}. By enforcing PDE constraints during training, PINNs ensure that the learned solutions are consistent with the structures arising from underlying physics. 
PINNs have demonstrated remarkable success in a wide range of applications, including stochastic control~\cite{pereira2021safe}, power systems\cite{misyris2020physics}, fluid dynamics~\cite{cai2022physics}, and medical applications~\cite{sahli2020physics}. Various extensions of PINNs have emerged, including Bayesian PINNs for noisy data~\cite{yang2021b}, PINNs with hard constraints for topology optimization~\cite{lu2021physics}, and parallel PINNs via domain decomposition for multi-scale problems~\cite{shukla2021parallel}. Additionally, recent theoretical analyses address approximation errors~\cite{sirignano2018dgm,grohs2023proof,darbon2020overcoming,darbon2021some}, generalization capabilities~\cite{de2022error,mishra2023estimates,mishra2022estimates,qian2023physics}, and convergence rates~\cite{fang2021high,pang2019fpinns,jiao2021rate}. When the physics models are uncertain, it is incorporated in a framework with an additional inverse problem that identifies governing equations~\cite{raissi2019physics}.
Our characterization of four types of long-term safety probabilities as the solutions to PDEs (Theorem~\ref{thm:InvariantProbability_MainTheorem1}-\ref{thm:ConvergenceProbability_MainTheorem4}) allows these quantities to be learned by physics-informed learning. To our knowledge, existing work has not previously employed PINNs to estimate long-term probabilities from short-term safe data while providing theoretical guarantees (such as Theorem~\ref{thm:full_pde_constraint} and Theorem~\ref{thm:pinn_convergence}). Our experiments demonstrate that embedding PDE constraints—specified using system parameters and with explicit gradient terms—facilitates stable interpolation across systems and gradient estimation. This results in additional advantages such as rapid online risk inference for varying system dynamics (Fig.~\ref{fig:varying parameter}) and stable computation of probability gradients (Fig.~\ref{fig:gradient}). 
}

\section{Preliminaries}
\label{sec:preliminaries}
Let $\R$, $\R_+$, $\R^n$, and $\R^{m\times n}$ be the set of real numbers, the set of non-negative real numbers, the set of $n$-dimensional real vectors, and the set of $m \times n$ real matrices. Let $x[k]$ be the $k$-th element of vector $x$. 
Let $x^+ = \max(x,0)$ and $x \wedge y = \min\{x,y\}$ for \(x, y \in \R\). Let $f:\gX \rightarrow \gY$ represent that $f$ is a mapping from space $\gX$ to space $\gY$. 
Let \(\mathbb{1}_{\gE}\) be an indicator function, which takes \(1\) when condition \(\gE\) holds and \(0\) otherwise. 
Let \(\mP_x(\gE) = \mP(\gE(X)|X_0=x)\) denote the probability of event \(\gE(X)\) involving a stochastic process \( X = \{X_t\}_{t\in\R_+}\) conditioned on
\(X_0=x\). Let \(\mE_x [ F ( X ) ] = \mE[ F ( X ) | X_0 = x ]\) denote the expectation of $F ( X )$ (a functional of $X$) conditioned on $X_0 = x$. We use upper-case letters (\eg$X$) to denote random variables and lower-case letters (\eg$x$) to denote their specific realizations.

\section{Problem Statement}
\label{sec:problem_statement}

We present the settings and the scope of this paper in section~\ref{sec:SystemDescription} and section~\ref{sec:objective}, respectively. 

\subsection{System Description and design specifications} 
\label{sec:SystemDescription}

\rev{
We consider a control system with stochastic noise of $k$-dimensional Brownian motion $W_t$ starting from $W_0 = 0$. The system state, $X_t \in \R^{n}, t \in \R_+$, evolves according to the following stochastic differential equation (SDE)
\begin{equation}\label{eq:x_trajectory}
    dX_t = f(X_t) dt + \sigma(X_t)dW_t,
\end{equation}  
where $\sigma: \mathbb{R}^n \rightarrow \mathbb{R}^k$ characterizes the magnitudes of the noise term. 
We restrict ourselves to the settings where $f,\sigma$ have sufficient regularity conditions such that the system has a well-defined solution. For example, the system is well-defined when $f$ and $\sigma$ are bounded and globally Lipschitz~\cite{oksendal_stochastic_2003a}, which are common assumptions usually satisfied in practice. More generally, conditions required to have a unique solution can be found in~\cite[Chapter~1]{oksendal_stochastic_2003a}, \cite[Chapter II.7]{borodin_stochastic_2017} and references therein. 
}
The size of $\sigma(X_t)$ is determined from the uncertainties in the disturbance, unmodeled dynamics~\cite{cheng2020safe}, and the prediction errors of the environmental variables~\cite{lefevre2014survey,ferguson2008}. Examples of these cases include when the unmodeled dynamics are captured using statistical models such as Gaussian Processes~\cite{cheng2020safe}\footnote{In such settings, the value of $\sigma(X_t)$ can be determined from the output of the Gaussian Process.} and when the motion of the environment variables are estimated using physics-based models such as Kalman filters~\cite{lefevre2014survey,ferguson2008}.
\rev{
Note that dynamics~\eqref{eq:x_trajectory} can capture not only autonomous systems, but also closed-loop systems with memoryless state feedback control (\eg PID and robust controllers).
}

In this paper, we aim to study the following two aspects related to safety for the system: forward invariance and recovery. The forward invariance refers to the property that the system state always stays in the safe region, while recovery refers to the property that the system state recovers to the safe region even if it originates from the unsafe region.
\begin{definition}[Safe Set]\label{def:safe_region}
The safe region is defined using a set $\gC$ that is characterized by a super level set of some function $\phi(x)$, \ie
\begin{equation}\label{eq:safe_region}
\mathcal{C} =\left\{x \in \mathbb{R}^{n}: \phi(x) \geq 0\right\},
\end{equation}
We call $\gC$ the \textit{safe set}, and \(\phi(x)\) the barrier function. Additionally, we denote the boundary of the safe set and unsafe set by
\begin{align}
    \partial \gC &= \{x\in\R^n:\phi(x)=0\}, \\    \gC^c & =\{x\in\R^{n}: \phi(x)<0\}.
\end{align}
\rev{
We assume that $\phi(x): \mathbb{R}^n \rightarrow \mathbb{R}$ is a first-order differentiable function whose gradient does not vanish at $\partial \gC$.}
\end{definition}
We assume that the dynamics and uncertainty in the system and the environment are all captured in~\eqref{eq:x_trajectory} so that the safe set can be defined using a static function $\phi(x)$ instead of a time-varying one.

\subsubsection{Forward invariance} When the initial state starts from a safe region, \ie$X_0 = x\in\gC$, we study the quantities associated with forward invariance: the safety margin from $\partial \gC$ and the first exit time from $\gC$.

Properties associated invariance to a set have been studied from the perspective of forward invariance. A forward invariant set with respect to a system is often defined to satisfy $X_{0} \in \gM  \Rightarrow X_t \in \gM, \forall t \in \R_+$ in a deterministic framework~\cite[Section 4.2]{khali1996adaptive}. However, in a stochastic system, it may not be always possible to ensure $X_t \in \gM$ at all time because the probability of exiting $\gM$ can accumulate over time.
Therefore, we use a modified definition, formally defined below, where $X_t \in \gM$ is only required to hold in a finite time interval $T$. 

\begin{definition}[Forward Invariance]
\label{def:forward-invariance}
The state of system \eqref{eq:x_trajectory} is said to remain forward invariant with respect to a set $\gM$ during a time interval $[0,T]$ if, given $X_0 = x \in \gM$, 
\begin{align}
    X_t \in \gM,\quad  \forall t \in (0, T].  
\end{align}
\end{definition}

We study the probability of forward invariance to the safe set through characterizing the distributions of the following two random variables:
\begin{align}
\label{eq:min_phi}
   \minf_x(T) &:= \inf\{ \phi(X_{t}) \in \R : t \in [0,T] , X_0 = x  \},\\
\label{eq:first_exit_time}
   \exit_x(\ell)  &:= \inf\{t \in \R_+ :  \phi(X_t)\leq \ell , X_0 = x \},
\end{align}
\noindent where $x \in \gC$, and $\ell \in \mathbb{R}$ is the safety margin. The value of $\minf_x(T)$ informs the worst-case safety margin from the boundary $\partial \gC$ during $[0, T]$, while the value of $\exit_x(0)$ is the first exit time from the safe set. 

The probability of staying within the safe set $\gC$ during a time interval $[0 , T]$ given $X_0 = x \in \gC$ can be computed using the distributions of \eqref{eq:min_phi} or \eqref{eq:first_exit_time} as
\begin{align} 
\label{eq:safe_prob_invariant}
\begin{split}
    \mP_x(X_t\in\gC, \forall t\in[0,T]) &= \mP \left(\minf_x(T) \geq 0 \right) \\
    &= 1 - \mP \left(\exit_x(0) \leq  T \right).
\end{split}
\end{align}

\subsubsection{Recovery} When the initial state is outside of the safe region, \ie$X_0 = x \notin \gC$, we study the quantities associated with recovery: the distance from $\gC$ and the recovery time. 

\begin{definition}[Forward Convergent Set]
A set $\gM$ is said to be a forward convergent set with respect to system~\eqref{eq:x_trajectory} if 
\begin{align}\label{eq:ProblemStatement_ConvergenceSet}
    X_{0} \notin \gM  \Rightarrow \exists \tau  \in (0, \infty) \text{ s.t. } X_{\tau} \in \gM .
\end{align}
\end{definition}
That is, for any initial state that does not belong to set $\gM$, the state re-enters the set in finite time. 
Equivalently, even if the state does not belong to set $\gM$ at some time, it will enter $\gM$ in finite time. Note that a forward convergent set does not require the state to stay in the set indefinitely after re-entry. 

We study forward convergence to the safe set through characterizing the distribution of the following two random variables:
\begin{align}
\label{eq:max_phi}
   \maxf_x(T) & := \sup\{ \phi(X_{t}) \in \R : t \in [0,T] , X_0 = x  \},\\
\label{eq:first_entry_time}
   \entrance_x(\ell) & := \inf\{t \in \R_+ :  \phi(X_t)\geq \ell , X_0 = x \},
\end{align}
\noindent where $x \notin \gC$. The value of $\maxf_x(T)$ informs whether the state re-enters the safe set within a time horizon $T$, and $\entrance_x(0)$ characterizes the duration required for the state to re-enter the safe region for the first time.

The probability of having re-entered the safe set $\gC$ by time $T$ given $X_0 = x \notin \gC$ can be computed using the distributions of \eqref{eq:max_phi} or \eqref{eq:first_entry_time} as  
\begin{align} 
\label{eq:safe_prob_convergent}
\begin{split}
 \mP_x\left( \exists t \in [0, T] \text{ s.t. } X_t \in\gC \right) &= \mP \left( \maxf_x(T) \geq 0 \right) \\
    &= 1 - \mP \left( \entrance_x(0) \leq  T \right),
\end{split}
\end{align}
given the fact that system~\eqref{eq:x_trajectory} is time-invariant.
More generally, the distributions of $\minf_x(T)$, $\maxf_x(T)$, \(\Gamma_x(\ell)\), and $\entrance_x(\ell)$ contain much richer information than the probability of forward invariance or recovery. Thus, knowing their exact distributions allows uncertainty and risk to be rigorously quantified in autonomous control systems.

\subsection{Objective and scope of this paper}
\label{sec:objective}
In this paper, we build upon analysis tools from stochastic processes to answer the following question:
\begin{itemize}
    \item When the stochastic system dynamics~\eqref{eq:x_trajectory} is given, what are the \textit{exact distributions} of~\eqref{eq:min_phi}, \eqref{eq:first_exit_time}, \eqref{eq:max_phi}, and~\eqref{eq:first_entry_time}?

    \item With the characterization of the exact distribution of~\eqref{eq:min_phi}, \eqref{eq:first_exit_time}, \eqref{eq:max_phi}, and~\eqref{eq:first_entry_time}, how do we efficiently calculate them?

\end{itemize}

We will answer the first question in section~\ref{sec:safety_thm} and the second question in section~\ref{sec:PINN}.
The distribution of \eqref{eq:min_phi} and~\eqref{eq:first_exit_time} will allow us to study the property of invariance from a variety of perspectives, such as the average failure time, the tail distributions of safety or loss of safety, and the mean and tail of the safety margin.
The distribution of \eqref{eq:max_phi} and~\eqref{eq:first_entry_time} will allow us to study the property of recovery from a variety of perspectives, such as the average recovery time, tail distribution of recovery, and the mean and tail distribution of recovery vs. crashes. 

\rev{
\begin{remark}
    The efficient calculation of~\eqref{eq:min_phi}, \eqref{eq:first_exit_time}, \eqref{eq:max_phi}, and~\eqref{eq:first_entry_time} can potentially be applied to chance-constrained safe predictive control~\cite{ono2015chance}, safe neural network-based control~\cite{wang2024myopically}, and reinforcement learning with long-term safety requirements~\cite{wachi2024long}.
\end{remark}
}


\section{Characterization of Probabilities}
\label{sec:safety_thm}

In this section, we study the properties associated with safety (forward invariance) and recovery (forward convergence). The former is quantified using the distribution of the safety margin~\eqref{eq:min_phi} and first exit time \eqref{eq:first_exit_time} in section~\ref{sec:complete_information_distribution}. The latter is quantified using the distribution of the distance to the safe set \eqref{eq:max_phi} and the recovery time \eqref{eq:first_entry_time} in section~\ref{sec:ForwardConvergence}. 

\rev{

\begin{definition}[Infinitesimal Generator]
The infinitesimal generator $A$ of
a stochastic process $\{ Y_t  \in \R^n \}_{t \in \R_+}$ is
\begin{equation}\label{eq:InfinitesimalGenerator}
\begin{split}
AF\left(y\right) & = \lim _{h\rightarrow 0} \frac{\mE_y\left[F(Y_{h})\right]-F\left(y\right)}{h}\\
\end{split}
\end{equation}
whose domain is the set of all functions $F: \R^n \rightarrow \R$ such that the limit of \eqref{eq:InfinitesimalGenerator} exists for all $y \in \R^n$. 
\end{definition}
Applying the infinitesimal generator to $\phi$ yields
\begin{align}
\label{eq:inf_gen_phi1}
       A\phi & = \gL_f\phi + \frac{1}{2}\tr\left(\sigma\sigma^\intercal\Hess{\phi}\right)\\
        & =: D_\phi(x)
\end{align}
We define $D_\phi(x)$ to be the value of \eqref{eq:inf_gen_phi1} evaluated at $x$. 
Here, $\gL_v$ is the Lie derivative along a vector or matrix field $v$.
The infinitesimal generator can be interpreted as the stochastic counterpart of the Lie derivative: It gives the derivative along the flow of $dX_t$ in expectation, \ie $\mE[ d\phi(X_t)] = A\phi(X_t)dt$.
}

The random variables of our interests, \eqref{eq:min_phi}, \eqref{eq:first_exit_time}, \eqref{eq:max_phi} and~\eqref{eq:first_entry_time},  are functions of $\phi(X_t)$. Since the stochastic dynamics of $d\phi(X_t)$ are driven by $X_t$, we consider the augmented state space
\begin{align}\label{eq:augumented_z}
        Z_t := \begin{bmatrix}\phi(X_t) \\ X_t
        \end{bmatrix}\in\R^{n+1}.
\end{align}
From It\^o's lemma, dynamics of $\phi(X_t)$ satisfies
\begin{equation}\label{eq:ito_lemma2}
\begin{split}
    d\phi(X_t)  = A\phi(X_t)dt + \gL_{\sigma}\phi(X_t)dW_t,
\end{split}
\end{equation}
where $\gL_{\sigma}$ is the Lie derivative along $\sigma$.
Combining \eqref{eq:x_trajectory} and~\eqref{eq:ito_lemma2}, the dynamics of $Z_t$ are given by the SDE 
\begin{align}\label{eq:MainTheorem_z}
    dZ_t &= \ave(Z_t)\,dt + \var(Z_t)\,dW_t,
\end{align}
where the drift and diffusion parameters are
\begin{align}\label{eq:mu_sigma_prime}
 & \ave(z) = \begin{bmatrix}
 D_{\phi}(x)\\
 f(x) 
\end{bmatrix}, && 
    \var(z) = \begin{bmatrix}
 \gL_\sigma\phi(x)\\
 \sigma(x)
\end{bmatrix}.
\end{align}

Since risk probabilities are complementary to safety probabilities, we will consider four safety-related probabilities of interest in section~\ref{sec:complete_information_distribution} and section~\ref{sec:ForwardConvergence}. We will call them risk probabilities for simplicity.

We then present four theorems on characterization of different risk probabilities.
The proof of all theorems can be found in the appendix of the paper.

\subsection{Probability of Forward Invariance}\label{sec:complete_information_distribution}

When the state initiates inside the safe set, the distribution of the safety margin, $\minf_x(T)$ in \eqref{eq:min_phi}, is given below.

\begin{theorem}\label{thm:InvariantProbability_MainTheorem1}
Consider system~\eqref{eq:x_trajectory} with the initial state \(X_0 = x\). 
Let $z = [\phi(x), x^\intercal]^\intercal\in\R^{n+1}$.
Then, the complementary cumulative distribution function (CCDF) of \(\minf_x(T)\), 
\addtocounter{equation}{1}
    \begin{align}\tag{\theequation.A}
    \label{eq:F_distribution}
        \FwdInv(z,T;\ell)=\mP(\Phi_x(T)\geq \ell),\quad\ell\in\R,
    \end{align}
    is the solution to the initial-boundary-value problem of the convection-diffusion equation on the super-level set \(\{z\in\R^n: z[1]\geq \ell\}\)
    \begin{align}\tag{\theequation.B}
    \label{eq:CauchyProblem}
        \begin{cases}
         {\partial \FwdInv\over\partial T} = {1\over 2}\grad\cdot(D\grad \FwdInv) + \gL_{\rho - {1\over 2}\grad\cdot D}\FwdInv,& z[1]\geq \ell, T>0,\\
         \FwdInv(z,T;\ell) = 0,& z[1]< \ell, T>0,\\
         \FwdInv(z,0;\ell) = \mathbb{1}_{\{z[1]\geq\ell\}}(z),&z\in\R^{n+1},
        \end{cases}
    \end{align}
    where \(D:=\zeta\zeta^\intercal\). Here, the spatial derivatives \(\grad{}\) and \(\gL\) apply to the variable \(z\in\R^{n+1}\).\footnote{This rule applies to Theorem 2 - 4 as well.}
\end{theorem}

Theorem~\ref{thm:InvariantProbability_MainTheorem1} converts the  problem of characterizing the distribution function into a deterministic convection-diffusion problem.
When we set \(\ell=0\), it gives the exact probability that $\gC$ is a forward invariant set with respect to~\eqref{eq:x_trajectory} during the time interval $[0,T]$. For arbitrary $\ell > 0$, it can be used to compute the probability of maintaining a safety margin of $\ell$ during $[0, T]$.

\rev{
\begin{remark}
In Theorem~\ref{thm:InvariantProbability_MainTheorem1} and the other theorems in this section, we use CDFs and CCDFs with a slight abuse of notation. The presence or absence of equality in the inequality conditions are chosen so that the obtained probability can be used to compute \eqref{eq:safe_prob_invariant} and \eqref{eq:safe_prob_convergent}. Also, note that if a random variable $Y$ has a PDF, then $\mP(Y\leq \ell) = \mP(Y< \ell)$, so the presence or absence of equality in the probability density function does not affect the computed probabilities.
\end{remark}
}

Moreover, the distribution of the first exit time from the safe set, $\exit_x(0)$ in \eqref{eq:first_exit_time}, and the first exit time from an arbitrary super level set \(\{x\in\R^n: \phi(x) \geq \ell\}\), is given below. 

\begin{theorem}\label{thm:InvariantProbability_MainTheorem2}
Consider system~\eqref{eq:x_trajectory} with the initial state \(X_0 = x\). Let $z = [\phi(x), x^\intercal]^\intercal\in\R^{n+1}$. Then, the cumulative distribution function (CDF) of the first exit time \(\exit_x(\ell)\),
\addtocounter{equation}{1}
\begin{align}\tag{\theequation.A}
\label{eq:G_distribution}
    \FwdInvExit(z,T; \ell) = \mP (\exit_x(\ell) \leq T ),
\end{align}
is the solution to
\begin{align*}\tag{\theequation.B}\label{eq:MainTheorem2}
\begin{cases}
    \frac{\partial \FwdInvExit}{\partial T} =  \frac{1}{2}\grad\cdot(D\grad{\FwdInvExit}) + \gL_{\rho-{1\over 2}\grad{ \cdot} D}\FwdInvExit
    ,& z[1]\geq \ell, T>0,\\
    \FwdInvExit(z,T;l) = 1,  & z[1] < \ell, T>0.\\
    \FwdInvExit(z,0;l) = \mathbb{1}_{\{z[1]<\ell\}}(z), & z\in\R^{n+1},\\
\end{cases}
\end{align*}
where $D = \zeta\zeta^\intercal$.
\end{theorem}
Similarly, Theorem~\ref{thm:InvariantProbability_MainTheorem2} gives the distribution of the first exit time from an arbitrary super level set of barrier function as the solution to a deterministic convection-diffusion equation. When we set \(\ell=0\), it can also be used to compute the exact probability of staying within the safe set during any time interval.

\subsection{Probability of Forward Convergence}
\label{sec:ForwardConvergence}
When the state initiates outside the safe set, the distribution of the distance from the safe set, \(\maxf_x(T)\) in~\eqref{eq:max_phi}, is given below. 
\begin{theorem}\label{thm:ConvergenceProbability_MainTheorem3}
Consider system~\eqref{eq:x_trajectory} with the initial state \(X_0 = x\). Let $z = [\phi(x), x^\intercal]^\intercal\in\R^{n+1}$. Then, the CDF of \(\maxf_x(T)\),   
\addtocounter{equation}{1}
\begin{equation}\tag{\theequation.A}
\label{eq:Q_distribution}
    \FwdConv(z,T;\ell) = \mP\left(\maxf_x(T) < \ell \right), \quad \ell \in \R,
\end{equation}
is the solution to
    \begin{align}\tag{\theequation.B}
    \label{eq:pde_thm3}
        \begin{cases}
         {\partial \FwdConv\over\partial T} = {1\over 2}\grad\cdot(D\grad \FwdConv) + \gL_{\rho - {1\over 2}\grad\cdot D}\FwdConv,& z[1] < \ell, T>0,\\
         \FwdConv(z,T;\ell) = 0,& z[1]\geq\ell, T>0,\\
         \FwdConv(z,0;\ell) = \mathbb{1}_{\{z[1]<\ell\}}(z),&z\in\R^{n+1},\\
        \end{cases}
    \end{align}
where $D = \var\var^\intercal$.
\end{theorem}
Theorem~\ref{thm:ConvergenceProbability_MainTheorem3} gives the distribution of safety distance \(\maxf_x(T)\) as a deterministic convection-diffusion problem. When we set \(\ell=0\), it provides the exact probability that a stochastic process~\eqref{eq:x_trajectory} initiating outside the safe set enters the safe set during the time interval $[0,T]$. When we set $\ell<0$, it can be used to compute the probability of coming close to $|\ell|$-distance from the safe set along the time interval $[0, T]$.

Finally, the distribution of the recovery time from the unsafe set, \(\entrance_x(0)\) in~\eqref{eq:first_entry_time}, and the entry time to an arbitrary super level set of the barrier function, \(\entrance_x(\ell)\), is given below. 
\begin{theorem}\label{thm:ConvergenceProbability_MainTheorem4}
Consider system~\eqref{eq:x_trajectory} with the initial state \(X_0 = x\). Let $z = [\phi(x), x^\intercal]^\intercal\in\R^{n+1}$. Then, the CDF of the recovery time \(\entrance_x(\ell)\), 
\addtocounter{equation}{1}
\begin{align}\tag{\theequation.A}
\label{eq:N_distribution}
    \FwdConvExit(z,T; \ell)=\mP ( \entrance_x(\ell) \leq T ),
\end{align}
is the solution to
\begin{align}\tag{\theequation.B}
\label{eq:PDE_D}
\begin{cases}
    \frac{\partial \FwdConvExit}{\partial t} =  \frac{1}{2}\grad\cdot(D\grad{\FwdConvExit}) + \gL_{\ave-{1\over 2} \grad{\cdot D}}{\FwdConvExit}
    ,& z[1] < \ell,T>0,\\
    \FwdConvExit(z,T;l) = 1,  & z[1]\geq \ell,T>0,\\
    \FwdConvExit(z,0;l) = \mathbb{1}_{\{z[1]\geq \ell\}}(z),& z\in\R^{n+1},\\
\end{cases}
\end{align}
where $D =\var\var^\intercal$.
\end{theorem}

Theorem~\ref{thm:ConvergenceProbability_MainTheorem4} gives the distribution of the re-entry time as a solution to a deterministic convection-diffusion equation. When we set \(\ell=0\), it can be used to compute the exact probability of entering the safe region during any time intervals. When we set $\ell < 0$, it can be used to study the first time to reach \(|\ell|\)-close to the safe set.

\rev{
Till here, we have characterized all probabilities of interest into the solution of certain PDEs. However, solving such PDEs with sampling-based methods~\cite{rubino2009rare} can be computationally expensive.
In the following section, we will present our proposed physics-informed learning framework for efficient evaluation of risk probabilities based on both sampled data and the PDE model we derived earlier.
}

\section{Physics-informed Learning}
\label{sec:PINN}

In this section, we present our proposed physics-informed learning framework to efficiently solve the associating PDEs for risk probabilities of interest.

\subsection{Physics-informed Learning Framework}

For simplicity of notation, we use one single notation $F$ to denote the probabilities of interest \ie equations~\eqref{eq:F_distribution},~\eqref{eq:G_distribution},~\eqref{eq:Q_distribution} and~\eqref{eq:N_distribution}, and use $W_F$ to denote the associated PDEs~\eqref{eq:CauchyProblem},~\eqref{eq:MainTheorem2},~\eqref{eq:pde_thm3} and~\eqref{eq:PDE_D}. For instance, for complementary cumulative distribution of \(\minf_x(T)\) defined in~\eqref{eq:F_distribution}, we have
\begin{equation}
\label{eq:safety_PDE}
    W_F := {\partial \FwdInv\over\partial T} - {1\over 2}\grad\cdot(D\grad \FwdInv) + \gL_{\rho - {1\over 2}\grad\cdot D}\FwdInv.
\end{equation}

While the PDE provides a way to get the actual risk probability of the system, to solve a PDE using numerical techniques is not easy in general, especially when the coefficients are time varying as in the case of~\eqref{eq:safety_PDE}.
Monte Carlo (MC) methods provide another way to solve this problem. Assume the dynamics of the system is given, one can simulate the system for an initial condition multiple times to get an empirical estimate of the risk probability by calculating the ratio of unsafe trajectories over all trajectories. However, 
MC requires huge number of trajectories to get accurate estimation, and the evaluation of the risk probability can only be conducted at a single point at a time.



\begin{figure*}[t]
    \centering
    \includegraphics[width=0.85\textwidth]{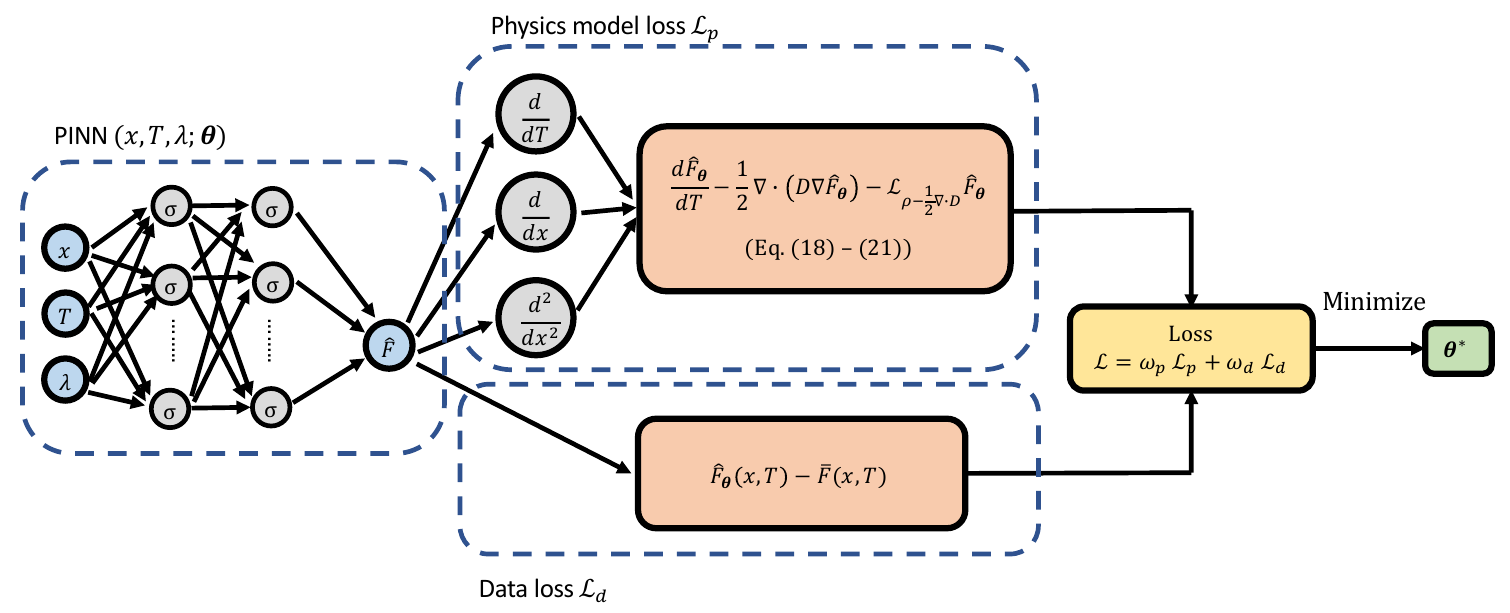}
    \caption{The training scheme of the physics-informed neural network. 
    }
    \label{fig:PINN diagram}
\end{figure*}

To leverage the advantages of PDE and MC and to overcome their drawbacks, we propose to use physics-informed neural networks (PINNs) to learn the mapping from the state and time horizon to the risk probability value $F$. Fig.~\ref{fig:PINN diagram} shows the architecture of the PINN. The PINN takes the state-time pair $(x,T)$ and the system parameter $\lambda$ as the input, and outputs the risk probability prediction $\hat{F}$, the state and time derivatives $ \frac{\partial \hat F}{\partial x}$ and $ \frac{\partial \hat F}{\partial T}$, and the Hessian $ \frac{\partial^2 \hat F}{\partial x^2}$, which come naturally from the automatic differentiation in deep learning frameworks such as PyTorch~\cite{paszke2019pytorch} and TensorFlow~\cite{abadi2016tensorflow}. Unlike standard PINN, we add the system parameter $\lambda$ as an input to achieve adaptations on varying system parameters.
Assume the PINN is parameterized by $\boldsymbol\theta$, the loss function is defined as
\begin{equation}
\label{eq:PINN overall loss function}
    \mathcal{L}(\boldsymbol\theta) = \omega_p \mathcal{L}_p(\boldsymbol\theta) + \omega_d \mathcal{L}_d(\boldsymbol\theta),
\end{equation}
where
\begin{equation}
\label{eq:PINN loss functions}
\begin{aligned}
    \mathcal{L}_p(\boldsymbol\theta) & = \frac{1}{|\mathcal{P}|} \sum_{(x,T) \in \mathcal{P}} \|W_{\hat{F}_{\boldsymbol\theta}}(x,T)\|_2^2, \\
    \mathcal{L}_d(\boldsymbol\theta) & = \frac{1}{|\mathcal{D}|} \sum_{(x,T) \in \mathcal{D}} \|\hat{F}_{\boldsymbol\theta}(x,T) - \bar{F}(x,T)\|_2^2.
\end{aligned}
\end{equation}
Here, $\bar{F}$ is the training data, $\hat{F}_{\boldsymbol\theta}$ is the prediction from the PINN, $\mathcal{P}$ and $\mathcal{D}$ are the training point sets for the physics model and external data, respectively. The loss function $\mathcal{L}$ consists of two parts, the physics model loss $\mathcal{L}_p$ and data loss $\mathcal{L}_d$. The physics model loss $\mathcal{L}_p$ measures the satisfaction of the PDE constraints for the learned output. It calculates the actual PDE equation value $W_{\hat{F}_{\boldsymbol\theta}}$, which is supposed to be $0$, and use its 2-norm as the loss. The data loss $\mathcal{L}_d$ measures the accuracy of the prediction of PINN on the training data. It calculates the mean square error between the PINN prediction and the training data point as the loss. The overall loss function $\mathcal{L}$ is the weighted sum of the physics model loss and data loss with weighting coefficients $\omega_p$ and $\omega_d$. In practice, we suggest setting a larger weight $\omega_d$ when sufficient data are available, and a smaller weight otherwise. Usually one can set both coefficients $\omega_p$ and $\omega_d$ close to 1 and increase the weight $\omega_p$ or $\omega_d$ if the physics model or data is more reliable.

The resulting PIPE framework combines MC data and the governing PDE into a PINN to learn the risk probability. The advantages of the PIPE framework include fast inference at test time, accurate estimation, and ability to generalize from the combination of data and model.

\subsection{Performance Analysis}
Here we provide performance analysis of PIPE. We first show that for standard neural networks (NNs) without physics model constraints, it is fundamentally difficult to estimate the risk probability of a longer time horizon than those generated from sampled trajectories. We then show that with the PINN, we are able to estimate the risk probability at any state for any time horizon with bounded error.
Let $\Omega$ be the state space, $\tau = [0,T_H]$ be the time domain, $\Sigma=(\partial \Omega \times[0, T_H]) \cup(\Omega \times\{0\})$ be the boundary of the space-time domain.
Denote $\D:=\Omega \times \tau$ for notation simplicity and denote $\Bar{\D}$ be the interior of $\D$.

\begin{corollary}
\label{cor:NN_worst_case_bound}
Suppose that $\D \in \mathbb{R}^{d+1}$ is a bounded domain, $u \in C^0(\bar{\D}) \cap C^2(\D)$ is the solution to the PDE of interest, and $\Tilde{u}(x,T), (x,T) \in \Sigma$ is the boundary condition. 
Let $\Sigma_s$ be a strict sub-region in $\Sigma$, and $\D_s$ be a strict sub-region in $\D$.
Consider a neural network ${F}_{\boldsymbol\theta}$ that is parameterized by 
$\boldsymbol\theta$ and has sufficient representation capabilities. For $\forall M > 0$, there can exist $\Bar{\boldsymbol\theta}$ that satisfies both of the following conditions simultaneously: 
\begin{enumerate}
    \item 
    $\sup _{(x, T) \in  \Sigma_s}|{F}_{\Bar{\boldsymbol\theta}}(x, T)-\Tilde{u}(x, T)|<\delta_1$
    \item
    $\sup _{(x, T) \in \D_s}|{F}_{\Bar{\boldsymbol\theta}}(x, T)-{u}(x, T) |<\delta_2$
\end{enumerate}
and
\begin{equation}
\label{eq:PINN_worst_case_bound}
    \sup _{(x, T) \in \D}\left|{F}_{\Bar{\boldsymbol\theta}}(x, T)-u(x, T)\right| \geq M.
\end{equation}   
\end{corollary}

\begin{proof}
    See Appendix~\ref{apx:proof_pinn}.
\end{proof}

\begin{theorem}
\label{thm:full_pde_constraint}
Suppose that $\D \in \mathbb{R}^{d+1}$ is a bounded domain, $u \in C^0(\bar{\D}) \cap C^2(\D)$ is the solution to the PDE of interest, and $\Tilde{u}(x,T), (x,T) \in \Sigma$ is the boundary condition. Let ${F}_{\boldsymbol\theta}$ denote a PINN parameterized by 
$\boldsymbol\theta$. If the following conditions holds:
\begin{enumerate}
    \item $\mathbb{E}_{\mathbf{Y}}\left[|{F}_{\boldsymbol\theta}(\mathbf{Y})-\Tilde{u}(\mathbf{Y})|\right]<\delta_1$, where $\mathbf{Y}$ is uniformly sampled from $\Sigma$
    \item $\mathbb{E}_{\mathbf{X}}\left[|W_{{F}_{\boldsymbol\theta}}(\mathbf{X})|\right]<\delta_2$, where $\mathbf{X}$ is uniformly sampled from $\D$
    \item 
    $F_{\boldsymbol\theta}$, $W_{{F}_{\boldsymbol\theta}}$, $u$ are $\frac{l}{2}$ Lipshitz continuous on $\D$.
\end{enumerate} 
Then the error of ${F}_{\boldsymbol\theta}$ over $\D$ is bounded by
\begin{equation}
\label{eq:pde_constraint_bound}
    \sup _{(x,T) \in \D}\left|{F}_{\boldsymbol\theta}(x, T)-u(x, T)\right| \leq \tilde \delta_1 + C \frac{\tilde \delta_2}{\sigma^2}
\end{equation}
where $C$ is a constant depending on $\D$, $\Sigma$ and $W$, and
\begin{equation}
\begin{aligned}
    \tilde \delta_1 & = \max \left\{\frac{2 \delta_1 |\Sigma| }{R_{\Sigma}|\Sigma|}, 2 l \cdot\left(\frac{\delta_1 |\Sigma| \cdot \Gamma(d / 2+1)}{l R_{\Sigma} \cdot \pi^{d / 2}}\right)^{\frac{1}{d+1}}\right\}, \\
    \tilde \delta_2 & = \max \left\{\frac{2\delta_2 |\D|}{R_{\D}|\D|}, 2 l \cdot\left(\frac{\delta_2 |\D| \cdot \Gamma((d+1) / 2+1)}{l R_{\D} \cdot \pi^{(d+1) / 2}}\right)^{\frac{1}{d+2}}\right\},
\end{aligned}
\end{equation}
with $R_{(\cdot)}$ being the regularity of $(\cdot)$, $\|(\cdot)\|$ is the Lebesgue measure
of a set $(\cdot)$ and $\Gamma$ is the Gamma function.
\end{theorem}

\begin{proof}
    See Appendix~\ref{apx:proof_pinn}.
\end{proof}


Corollary~\ref{cor:NN_worst_case_bound} says that standard NN can give arbitrarily inaccurate prediction due to insufficient physical constraints. This explains why risk estimation problems cannot be handled solely on fitting sampled data. 
Theorem~\ref{thm:full_pde_constraint} says that when the PDE constraint is imposed on the full space-time domain, the prediction of the PINN has bounded error.


Also, we present the following theorem based on~\cite{shin2020convergence} to show that the PINN converges to the solution of the safety-related PDE when the number of training data increases under some additional regularity assumptions.
\begin{assumption}
\label{asp:dynamics}
    The system dynamics $f$, $g$ and $\sigma$ in~\eqref{eq:x_trajectory} are bounded continuous functions. Also, for any $x \in \mathcal{X}$, $\text{eig}(\sigma(x)) \neq 0$ and $\sigma^{11}(x) > 0$, \ie all eigenvalues of $\sigma$ is non-zero and the first row first column entry of the noise magnitude is positive.
\end{assumption}
\begin{assumption}
\label{asp:domain}

    Let $D \in \mathbb{R}^{d+1}$ and $\Gamma \in \mathbb{R}^{d}$ be the state-time space of interest and a subset of its boundary, respectively.
    Let $\mu_r$ and $\mu_b$ be probability distributions defined on $D$ and $\Gamma$.
    Let $\rho_r$ be the probability density of $\mu_r$ with respect to $(d+1)$-dimensional Lebesgue measure on $D$. Let $\rho_b$ be the probability density of $\mu_b$ with respect to $d$-dimensional Hausdorff measure on $\Gamma$.
    We assume the following conditions.
    \begin{enumerate}

    \item $D$ is at least of class $C^{0,1}$.

    \item  $\rho_r$ and $\rho_b$ are supported on $D$ and $\Gamma$, respectively. Also, $\inf _D \rho_r>0$ and $\inf _{\Gamma} \rho_b>0$.
    
    \item For $\epsilon>0$, there exists partitions of $D$ and $\Gamma,\left\{D_j^\epsilon\right\}_{j=1}^{K_r}$ and $\left\{\Gamma_j^\epsilon\right\}_{j=1}^{K_b}$ that depend on $\epsilon$ such that for each $j$, there are cubes $H_\epsilon\left(\mathbf{z}_{j, r}\right)$ and $H_\epsilon\left(\mathbf{z}_{j, b}\right)$ of side length $\epsilon$ centered at $\mathbf{z}_{j, r} \in D_j^\epsilon$ and $\mathbf{z}_{j, b} \in \Gamma_j^\epsilon$, respectively, satisfying $D_j^\epsilon \subset H_\epsilon\left(\mathbf{z}_{j, r}\right)$ and $\Gamma_j^\epsilon \subset H_\epsilon\left(\mathbf{z}_{j, b}\right)$.

    \item There exists positive constants $c_r, c_b$ such that $\forall \epsilon>0$, the partitions from the above satisfy $c_r \epsilon^d \leq \mu_r\left(D_j^\epsilon\right)$ and $c_b \epsilon^{d-1} \leq \mu_b\left(\Gamma_j^\epsilon\right)$ for all $j$.
    There exists positive constants $C_r, C_b$ such that $\forall x_r \in D$ and $\forall x_b \in \Gamma, \mu_r\left(B_\epsilon\left(x_r\right) \cap D\right) \leq C_r \epsilon^d$ and $\mu_b\left(B_\epsilon\left(x_b\right) \cap \Gamma\right) \leq C_b \epsilon^{d-1}$ where $B_\epsilon(x)$ is a closed ball of radius $\epsilon$ centered at $x$.
    Here $C_r, c_r$ depend only on $\left(D, \mu_r\right)$ and $C_b, c_b$ depend only on $\left(\Gamma, \mu_b\right)$.

    \item When $d=1$, we assume that all boundary points are available. Thus, no random sample is needed on the boundary.
        
    \item For $x^{\prime} \in \partial D$, there exists a closed ball $\bar{B}$ in $\mathbb{R}^d$ such that $\bar{B} \cap \bar{D}=\left\{x^{\prime}\right\}$.
    \end{enumerate}
\end{assumption}
Note that Assumption~\ref{asp:dynamics} is not strict for most 
systems.\footnote{the noise magnitude requirement can be easily satisfied by rearranging the sequence of the state if the first state has zero noise.} 
\rev{
Assumption~\ref{asp:domain} guarantees that random samples drawn from probability distributions can fill up both the interior of the domain $D$ and the boundary $\partial D$. These are mild assumptions and can be satisfied in many practical cases. For example, let $D=(0,1)^d$, then the uniform probability distributions on both $D$ and $\partial D$ satisfy Assumption~\ref{asp:domain}.
}
\begin{theorem}
\label{thm:pinn_convergence}
    Suppose Assumption~\ref{asp:dynamics} holds, and $\D \in \mathbb{R}^{d+1}$ is a bounded domain that satisfies Assumption~\ref{asp:domain}, $u \in C^0(\bar{\D}) \cap C^2(\D)$ is the solution to the PDE of interest.
    Let $m_r$ be the number of training data in $D$. Let $F_{m_r}$ be the physics-informed neural networks that minimize the losses~\eqref{eq:PINN overall loss function} with $m_r$ training data. Assume for any $m_r$, architecture-wise the physics-informed neural network has enough representation power to characterize $u$, then
    \begin{equation}
        \lim _{m_r \rightarrow \infty} F_{m_r}=u.
    \end{equation}
\end{theorem}
\begin{proof}
    See Appendix~\ref{apx:proof_pinn_cont}.
\end{proof}

Till here, we have shown error bounds and convergence properties of the PIPE framework.
For approximation error of PINNs in terms of neural network sizes (number of layers and number of neurons per layer), previous works on universal approximation theorem have shown that sufficiently large PINNs can approximate PDEs uniformly~\cite{sirignano2018dgm, grohs2023proof, darbon2020overcoming, darbon2021some}.

\section{Experiments}
\label{sec:experiments}
We conduct four experiments to illustrate the efficacy of the proposed method. The system dynamics of interest is~\eqref{eq:x_trajectory}
with $X \in \mathbb{R}$, $f(X) \equiv \lambda$, $g(X) \equiv 0$ and $\sigma(X) \equiv 1$. The system dynamics become
\begin{equation}
\label{eq:experiment system dynamics}
    dX_t = \lambda \: dt + dW_t.
\end{equation}
The safe set is defined as~\eqref{eq:safe_region} with $\phi(x) = x-2$. The state-time region of interest is $\Omega \times \tau = [-10,2] \times [0,10]$.
For risk probability, we consider the recovery probability of the system from initial state $x_0 \notin \mathcal{C}$ outside the safe set. Specifically, from section~\ref{sec:safety_thm} we know that the risk probability $F$ is characterized by the solution of the following convection diffusion equation
\begin{equation}
\label{eq:cdc recovery pde}
    \frac{\partial F}{\partial T}(x, T) 
    = \lambda \frac{\partial F}{\partial x}(x, T) + \frac{1}{2} \operatorname{tr}\left( \frac{\partial^{2} F}{\partial x^{2}}(x, T) \right),
\end{equation}
with initial condition $F(x,0) = \mathds{1}(x \geq 2)$ and boundary condition $F(2,T) = 1$.
We choose this system because we have the analytical solution of~\eqref{eq:cdc recovery pde} as ground truth for comparison, as given by
\begin{equation}
    F(x,T) = \int_0^T \frac{(2-x)}{\sqrt{2 \pi t^{3}}} \exp \left(-\frac{\left((2-x)-\lambda t\right)^2}{2 t}\right) dt.
\end{equation}
The empirical data of the risk probability is acquired by running MC with the system dynamics~\eqref{eq:x_trajectory} with initial state $x=x_0$ multiple times, and calculate ratio of trajectories with recovered state, \ie 
\begin{equation}
    \bar{F}(x,T) = \pr(\exists t \in [0,T], x_t \in \mathcal{C} \mid x_0 = x) = \frac{N_{\text{recovery}}}{N},
\end{equation}
where $N_{\text{recovery}}$ is the number of trajectories that the state recovers to the safe set during time horizon $[0,T]$, and $N$ is the total number of sample trajectories and is a tunable parameter that affects the accuracy of the estimated risk probability. 
Specifically, larger $N$ gives more accurate estimation.

In all experiments, we use PINN with 3 hidden layers and 32 neurons per layer. The activation function is chosen as hyperbolic tangent function ($\tanh$). We use Adam optimizer~\cite{kingma2014adam} for training with initial learning rate set as $0.001$. The PINN parameters $\boldsymbol\theta$ is initialized via Glorot uniform initialization. The weights in the loss function~\eqref{eq:PINN overall loss function} are set to be $\omega_p = \omega_d = 1$. We train the PINN for 60000 epochs in all experiments. The simulation is constructed based on the DeepXDE framework~\cite{lu2021deepxde}. 
Experiment details, simulation results on higher-dimensional systems, and applications to stochastic safe control can be found in the 
Appendix.
Code is available at~\href{https://github.com/jacobwang925/PIPE-L4DC}{https://github.com/jacobwang925/PIPE-L4DC}.


\begin{figure}
    \centering
    \includegraphics[width=0.43\textwidth]{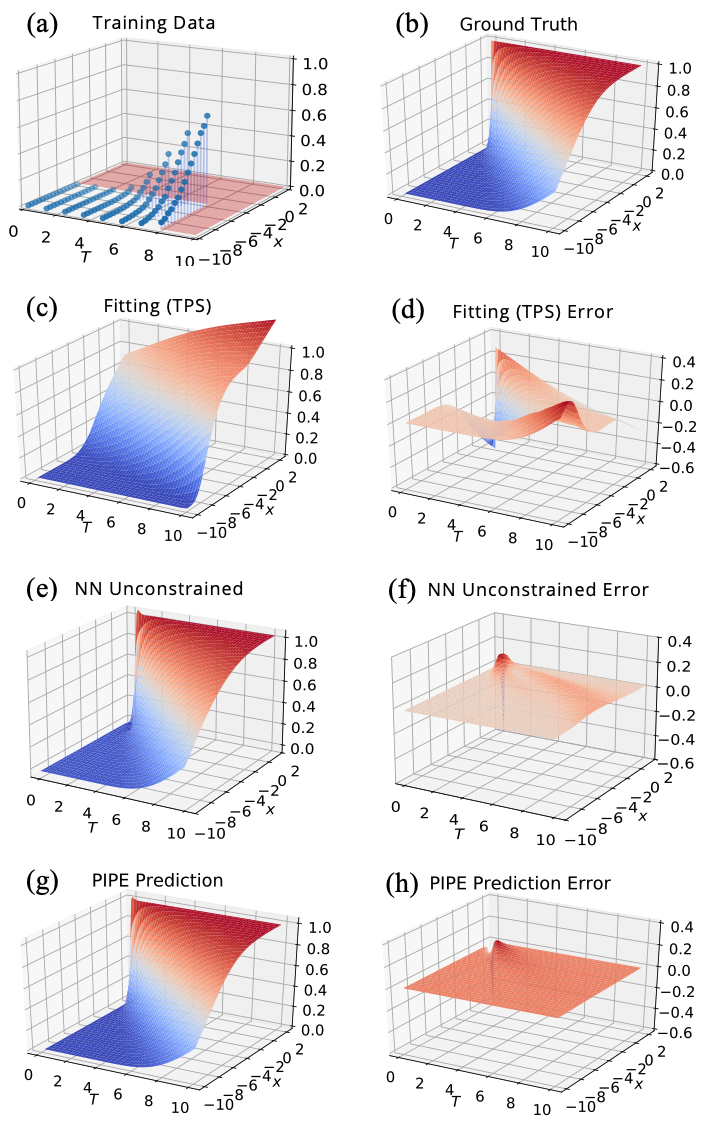}
    \caption{Settings and results of the risk probability generalization task. Note that shaded area in (a) in the spatial-temporal space are not covered by training data. MC with fitting (TPS), neural network without PDE constraint and the proposed PIPE framework are compared. The average absolute error of prediction is $9.2 \times 10^{-2}$ for TPS, and $1.5 \times 10^{-2}$ for neural network without PDE constraints, $0.3 \times 10^{-2}$ for PIPE.}
    \label{fig:generalization}
\end{figure}


\subsection{Generalization to unseen regions}
\label{sec:generalization}
In this experiment, we test the generalization ability of PIPE to unseen regions of the state-time space. We consider system~\eqref{eq:experiment system dynamics} with $\lambda=1$. We train the PINN with data only on the sub-region of the state-time space $\Omega \times \tau = [-10,-4] \times [0,8]$, but test the trained PINN on the full state-time region $\Omega \times \tau = [-10,2] \times [0,10]$.
The training data is acquired through MC with sample trajectory number $N = 1000$, and is down-sampled to $dx = 0.4$ and $dT = 1$. For comparison, we consider MC with fitting and neural networks without the PDE constraint. 
Among all fitting methods (\eg cubic spline and polynomial fitting), thin plate spline (TPS) performs the best and is used for presentation.
Fig.~\ref{fig:generalization} visualizes the training data samples and shows the results. The spline fitting and the standard neural network do not include any physical model constraint, thus fail to accurately generalize to unseen regions in the state space. In contrast, PIPE can infer the risk probability value very accurately in the entire state-temporal space due to the global enforcement of physics loss, which constrains the prediction to satisfy the PDE models we derived in section~\ref{sec:safety_thm}. 

\subsection{Efficient estimation of risk probability}
\label{sec:estimation}
In this experiment, we show that PIPE will give more efficient estimations of risk probability in terms of accuracy and sample number compared to MC and its variants. We consider system~\eqref{eq:experiment system dynamics} with $\lambda=1$. The training data is sampled on the state-time space $\Omega \times \tau = [-10,2] \times [0,10]$ with $dx = 0.2$ and $dT=0.1$.
We compare the risk probability estimation error of PIPE and MC, on two regions in the state-time space:
\begin{enumerate}
    \item Normal event region: $\Omega \times \tau = [-6,-2] \times [4,6]$, where the average probability is $0.412$.
    \item Rare event region: $\Omega \times \tau = [-2,0] \times [8,10]$, where the average probability is $0.985$.
\end{enumerate}
For fairer comparison, we use a uniform filter of kernel size $3$ on the MC data to smooth out the noise, as the main cause of inaccuracy of MC estimation is sampling noise. Fig.~\ref{fig:PINN MC log plot} shows the percentage errors of risk probability inference under different MC sample numbers $N$.
As the sample number goes up, prediction errors for all three approaches decrease. The denoised MC has lower error compared to standard MC as a result of denoising, and their errors tend to converge since the sampling noise contributes less to the error as the sample number increases. On both rare events and normal events, PIPE yields more accurate estimation than MC and denoised MC across all sample numbers. This indicates that PIPE has better sample efficiency than MC and its variants, as it requires less sample data to achieve the same prediction accuracy.
This desired feature of PIPE is due to the fact that it incorporates model knowledge into the MC data to further enhance its accuracy by taking the physics-informed neighboring relationships of the data into consideration.
\begin{figure}
    \centering
    \includegraphics[width=0.4\textwidth]{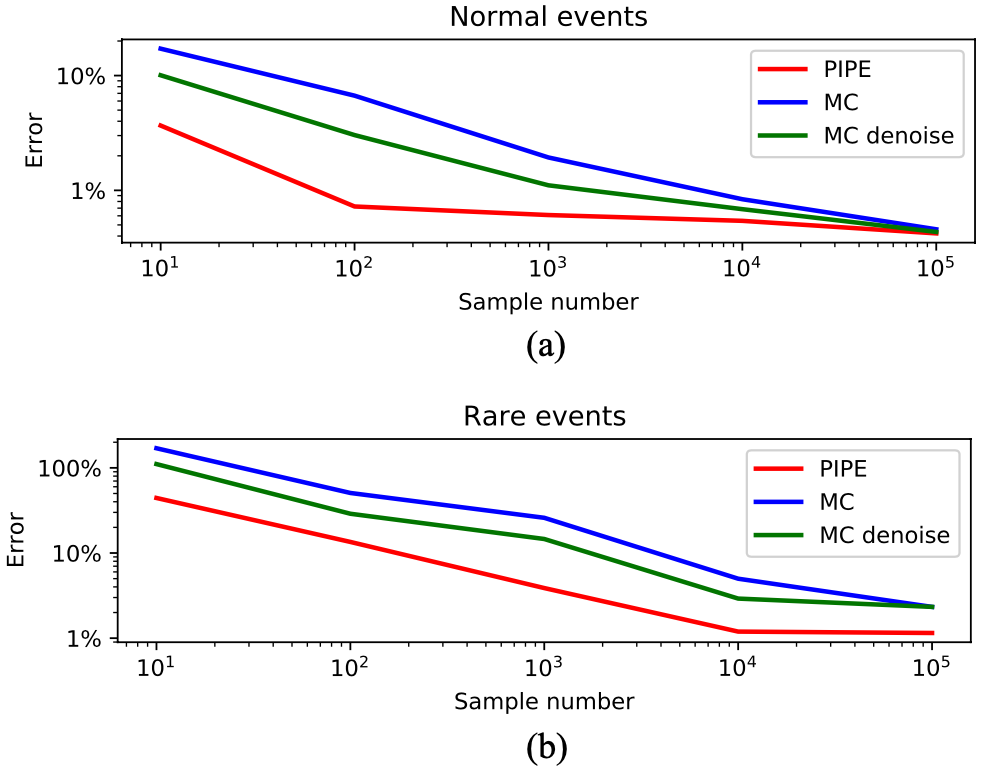}
    \caption{Percentage error of risk probability estimation for different MC sample numbers for (a) rare events and (b) normal events. PIPE, MC and denoised MC with uniform kernel filtering are compared. Both error and sample number are in log scale.}
    \label{fig:PINN MC log plot}
\end{figure}

\begin{figure}
    \centering
    \includegraphics[width=0.43\textwidth]{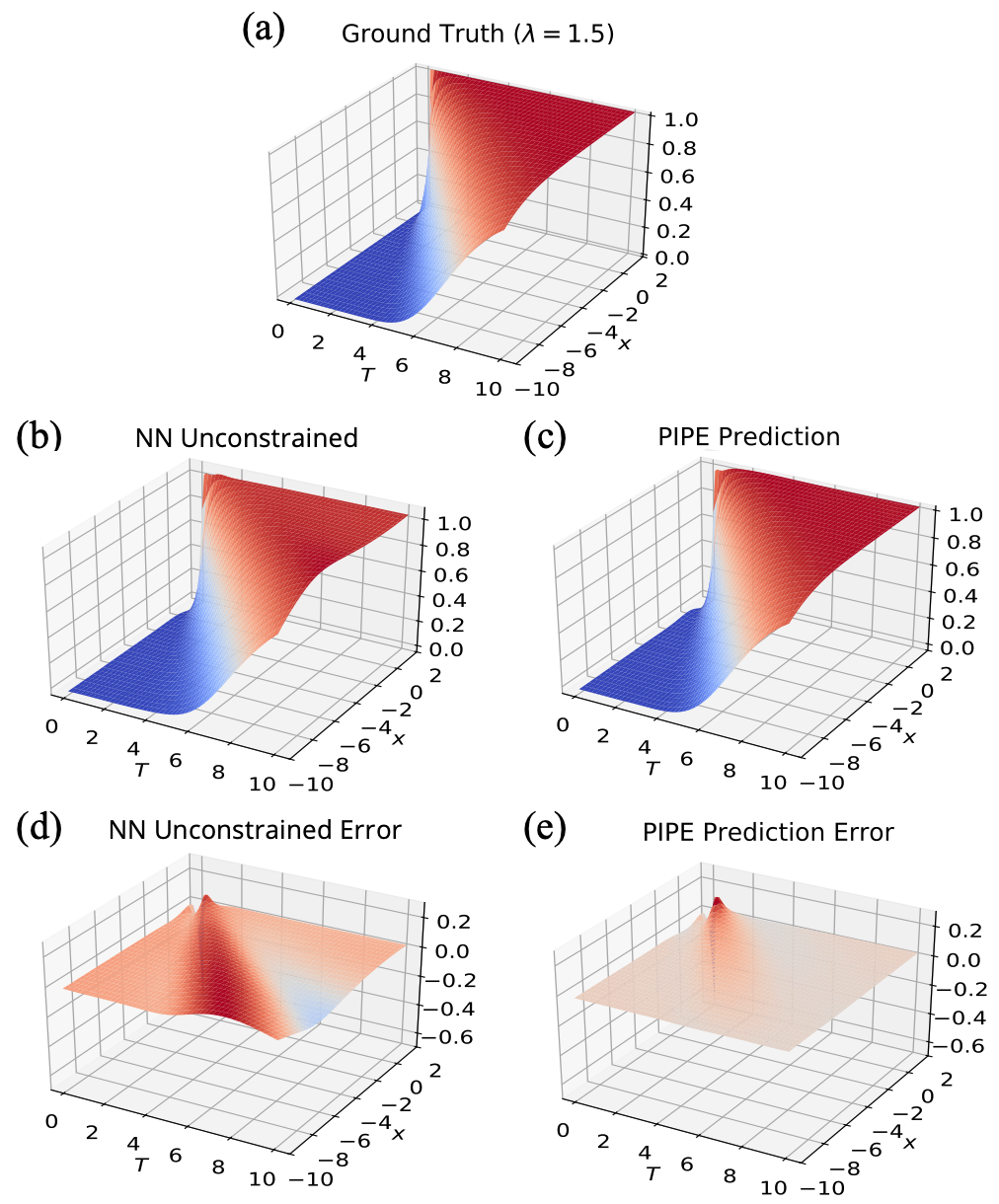}
    \caption{Risk probability prediction of unconstrained neural network and PIPE on unseen system parameters $\lambda_{\text{test}} = 1.5$. The average absolute prediction error across all $\lambda_{\text{test}} = [0.3, 0.7, 1.2, 1.5, 2]$ is $2.23 \times 10^{-2}$ for unconstrained neural network, and $0.60 \times 10^{-2}$ for PIPE.}
    \label{fig:varying parameter}
\end{figure}

\rev{
The computation time for different methods are reported in Table~\ref{tab:method_timing}. It can be seen that, although PIPE requires offline training, it offers fast online inference, enabling real-time evaluation for risk probabilities. In addition, the training time is independent of the precision of the training data. MC has no training phase, but incurs a significant online cost when the number of samples is high to achieve the desired accuracy. 
}

\rev{
\begin{table}
\centering
\begin{tabular}{|l|c|c|}
\hline
\textbf{Method} & \textbf{Offline (Training)} & \textbf{Online (Inference)} \\
\hline
PIPE (GPU, L4) & 62.26 s  & 0.0013 s \\
PIPE (CPU)     & 262.79 s & 0.0012 s \\
MC ($10^1$ samples) & N/A      & 0.0026 s \\
MC ($10^5$ samples) & N/A      & 2.53 s \\
\hline
\end{tabular}
\caption{\rev{Comparison of computation time for different methods. The inference time is calculated for evaluating risk probability of one initial state with fixed time horizon $T = 10$.}}
\label{tab:method_timing}
\end{table}
}

\subsection{Adaptation on changing system parameters}
\label{sec:adaptation}
In this experiment, we show that PIPE will allow generalization to uncertain parameters of the system. We consider system~\eqref{eq:experiment system dynamics} with varying $\lambda\in [0,2]$. 
We use MC data with sample number $N=10000$ for a fixed set of $\lambda_{\text{train}} = [0.1,0.5,0.8,1]$ for training, and test PIPE after training on $\lambda_{\text{test}} = [0.3, 0.7, 1.2, 1.5, 2]$. 
Fig.~\ref{fig:varying parameter} visualizes the results testing instance $\lambda_{\text{test}} = 1.5$.
We can see that PIPE is capable of accurately predicting risk probability for systems with unseen and out-of-distribution parameters beyond training, while a neural network without PDE constraints fails to do so. In the prediction error plot, the only region where PIPE has larger prediction error is at $T=0$ and $x \in \partial \mathcal{C}$ on the boundary of the safe set. This is because the risk probability at this point is not well defined (it can be either $0$ or $1$), and this point will not be considered in a control scenario as we are always interested in long-term safety where $T \gg 0$. This adaptation feature of the PIPE framework indicates its potential use on stochastic safe control with uncertain system parameters, and it also opens the door for physics-informed learning on a family of PDEs. In general, PDEs with different parameters can have qualitatively different behaviors, so is hard to generalize. The control theory model allows us to have a sense when the PDEs are qualitatively similar with different parameters, and thus allows generalization within the qualitatively similar cases. 

\subsection{Estimating the gradient of risk probability}
\label{sec:gradient}
In this experiment, we show that PIPE is able to generate accurate gradient predictions of risk probabilities. We consider system~\eqref{eq:experiment system dynamics} with $\lambda=1$. Similar to the generalization task, we train the PINN with MC data of $N=1000$ on the sub-region $\Omega \times \tau = [-10,-4] \times [0,8]$ and test the trained PINN on the full state-time region $\Omega \times \tau = [-10,2] \times [0,10]$. We then take the finite difference of the risk probability with regard to the state $x$ to calculate its gradient, for ground truth $F$, MC estimation $\bar{F}$ and PIPE prediction $\hat{F}_{\boldsymbol\theta}$ with and without PDE constraints.\footnote{Note that for PIPE, automatic differentiation~\cite{baydin2018automatic} can also be used to obtain risk probability gradients, and we did not find qualitative difference with the finite difference methods for this specific experiment.} Fig.~\ref{fig:gradient} shows the results. It can be seen that the MC gradient estimate is rather noisy, and the unconstrained NN does not provide accurate gradient estimation beyond the region of training samples. In contrast, PIPE gives a much more accurate gradient estimation, because it incorporates the physics model information inside the training process. It is also worth noting that PIPE does not use any governing laws of the risk probability gradient during training, and by considering the risk probability PDE alone, it can provide very accurate estimations of the gradient. The results indicate that PIPE can enable the usage of a lot of first- and higher-order stochastic safe control methods online, by providing accurate and fast estimation of the risk probability gradients.

\begin{figure}
    \centering
    \includegraphics[width=0.43\textwidth]{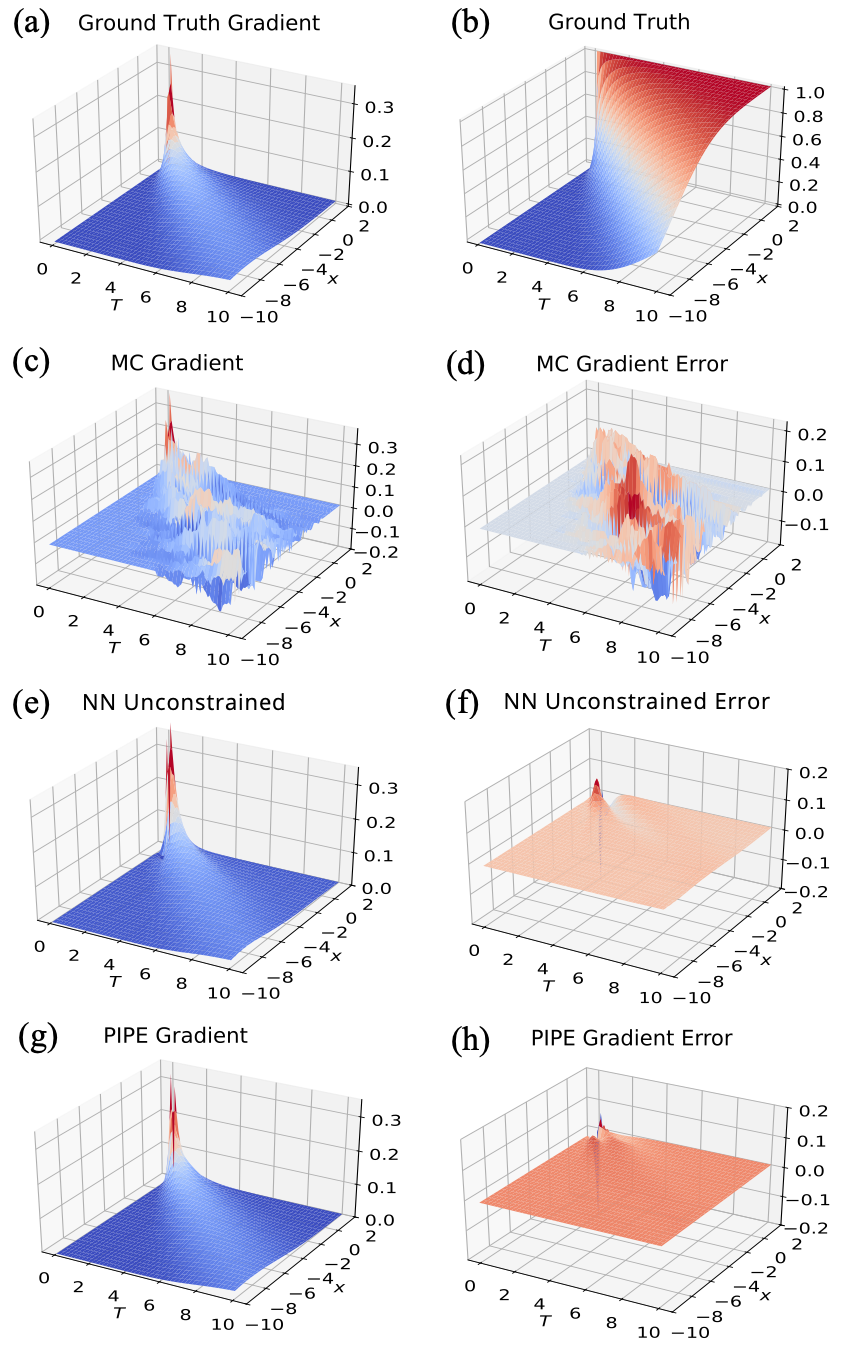}
    \caption{Gradient of the risk probability prediction of MC, unconstrained NN and PIPE. The average absolute error of gradient prediction is $2.78 \times 10^{-2}$ for MC, $0.17 \times 10^{-2}$ for NN without constraints, and $0.06 \times 10^{-2}$ for PIPE.}
    \label{fig:gradient}
\end{figure}

\section{Conclusion}
\label{sec:conclusion}
In this paper, we propose physics-informed probability estimator (PIPE), a framework to efficiently and accurately estimate long-term risk probabilities in stochastic systems. We first characterize the exact probability distributions of the minimum and maximum barrier function values in a time interval and the first entry and exit times to and from its super level sets. The distributions of these variables are then used to characterize the probability of safety and recovery, the safety margin, and the mean and tail distributions of the failure and recovery times into the solution of certain partial differential equations (PDEs).
We then propose a physics-informed learning approach to solve the corresponding PDEs and efficiently estimate the risk probabilities. The method combines data from rare event simulation and the underlying governing PDE we derived, to accurately learn the risk probability as well as its gradient on the entire state-temporal domain of interest. Experiments show better sample efficiencies of the physics-informed learning framework compared to MC, and the ability to generalize to unseen regions in the state-temporal space beyond training. The framework is also robust to uncertain parameters in the system dynamics, and can infer risk probability values of a class of systems with training data only from a fixed number of systems, which provides key foundations for first- and higher-order methods for stochastic safe control. 
Future work includes extending the current framework to inverse learning~\cite{raissi2019physics} and Bayesian learning~\cite{yang2021b} settings, to account for systems with unknown dynamics where the PDEs for risk characterization are uncertain, and to real-world high-dimensional systems~\cite{hoshino2023scalable, wang2024physics}.

\bibliography{citation}

\begin{thebibliography}{10}
\providecommand{\url}[1]{#1}
\csname url@samestyle\endcsname
\providecommand{\newblock}{\relax}
\providecommand{\bibinfo}[2]{#2}
\providecommand{\BIBentrySTDinterwordspacing}{\spaceskip=0pt\relax}
\providecommand{\BIBentryALTinterwordstretchfactor}{4}
\providecommand{\BIBentryALTinterwordspacing}{\spaceskip=\fontdimen2\font plus
\BIBentryALTinterwordstretchfactor\fontdimen3\font minus \fontdimen4\font\relax}
\providecommand{\BIBforeignlanguage}[2]{{%
\expandafter\ifx\csname l@#1\endcsname\relax
\typeout{** WARNING: IEEEtran.bst: No hyphenation pattern has been}%
\typeout{** loaded for the language `#1'. Using the pattern for}%
\typeout{** the default language instead.}%
\else
\language=\csname l@#1\endcsname
\fi
#2}}
\providecommand{\BIBdecl}{\relax}
\BIBdecl

\bibitem{mcneil2015quantitative}
A.~J. McNeil, R.~Frey, and P.~Embrechts, \emph{Quantitative risk management: concepts, techniques and tools-revised edition}.\hskip 1em plus 0.5em minus 0.4em\relax Princeton university press, 2015.

\bibitem{prajna2007framework}
S.~Prajna, A.~Jadbabaie, and G.~J. Pappas, ``A framework for worst-case and stochastic safety verification using barrier certificates,'' \emph{IEEE Transactions on Automatic Control}, vol.~52, no.~8, pp. 1415--1428, 2007.

\bibitem{yaghoubi2020risk}
S.~Yaghoubi, K.~Majd, G.~Fainekos, T.~Yamaguchi, D.~Prokhorov, and B.~Hoxha, ``Risk-bounded control using stochastic barrier functions,'' \emph{IEEE Control Systems Letters}, 2020.

\bibitem{santoyo2021barrier}
C.~Santoyo, M.~Dutreix, and S.~Coogan, ``A barrier function approach to finite-time stochastic system verification and control,'' \emph{Automatica}, vol. 125, p. 109439, 2021.

\bibitem{cheng2020safe}
R.~Cheng, M.~J. Khojasteh, A.~D. Ames, and J.~W. Burdick, ``Safe multi-agent interaction through robust control barrier functions with learned uncertainties,'' in \emph{2020 59th IEEE Conference on Decision and Control (CDC)}.\hskip 1em plus 0.5em minus 0.4em\relax IEEE, 2020, pp. 777--783.

\bibitem{tong2022optimization}
S.~Tong, A.~Subramanyam, and V.~Rao, ``Optimization under rare chance constraints,'' \emph{SIAM Journal on Optimization}, vol.~32, no.~2, pp. 930--958, 2022.

\bibitem{rubino2009rare}
G.~Rubino and B.~Tuffin, \emph{Rare event simulation using Monte Carlo methods}.\hskip 1em plus 0.5em minus 0.4em\relax John Wiley \& Sons, 2009.

\bibitem{janssen2013monte}
H.~Janssen, ``Monte-carlo based uncertainty analysis: Sampling efficiency and sampling convergence,'' \emph{Reliability Engineering \& System Safety}, vol. 109, pp. 123--132, 2013.

\bibitem{cerou2012sequential}
F.~C{\'e}rou, P.~Del~Moral, T.~Furon, and A.~Guyader, ``Sequential monte carlo for rare event estimation,'' \emph{Statistics and computing}, vol.~22, no.~3, pp. 795--808, 2012.

\bibitem{uribe2021cross}
F.~Uribe, I.~Papaioannou, Y.~M. Marzouk, and D.~Straub, ``Cross-entropy-based importance sampling with failure-informed dimension reduction for rare event simulation,'' \emph{SIAM/ASA Journal on Uncertainty Quantification}, vol.~9, no.~2, pp. 818--847, 2021.

\bibitem{yang2024rare}
D.~T. Yang, A.~M. Goldberg, and L.~T. Chong, ``Rare-event sampling using a reinforcement learning-based weighted ensemble method,'' \emph{bioRxiv}, 2024.

\bibitem{ehre2023stein}
M.~Ehre, I.~Papaioannou, and D.~Straub, ``Stein variational rare event simulation,'' \emph{arXiv preprint arXiv:2308.04971}, 2023.

\bibitem{zuev2015subset}
K.~Zuev, ``Subset simulation method for rare event estimation: an introduction,'' \emph{arXiv preprint arXiv:1505.03506}, 2015.

\bibitem{raissi2019physics}
M.~Raissi, P.~Perdikaris, and G.~E. Karniadakis, ``Physics-informed neural networks: A deep learning framework for solving forward and inverse problems involving nonlinear partial differential equations,'' \emph{Journal of Computational physics}, vol. 378, pp. 686--707, 2019.

\bibitem{han2018solving}
J.~Han, A.~Jentzen, and W.~E, ``Solving high-dimensional partial differential equations using deep learning,'' \emph{Proceedings of the National Academy of Sciences}, vol. 115, no.~34, pp. 8505--8510, 2018.

\bibitem{cuomo2022scientific}
S.~Cuomo, V.~S. Di~Cola, F.~Giampaolo, G.~Rozza, M.~Raissi, and F.~Piccialli, ``Scientific machine learning through physics--informed neural networks: Where we are and what’s next,'' \emph{Journal of Scientific Computing}, vol.~92, no.~3, p.~88, 2022.

\bibitem{pereira2021safe}
M.~Pereira, Z.~Wang, I.~Exarchos, and E.~Theodorou, ``Safe optimal control using stochastic barrier functions and deep forward-backward sdes,'' in \emph{Conference on Robot Learning}.\hskip 1em plus 0.5em minus 0.4em\relax PMLR, 2021, pp. 1783--1801.

\bibitem{misyris2020physics}
G.~S. Misyris, A.~Venzke, and S.~Chatzivasileiadis, ``Physics-informed neural networks for power systems,'' in \emph{2020 IEEE Power \& Energy Society General Meeting (PESGM)}.\hskip 1em plus 0.5em minus 0.4em\relax IEEE, 2020, pp. 1--5.

\bibitem{cai2022physics}
S.~Cai, Z.~Mao, Z.~Wang, M.~Yin, and G.~E. Karniadakis, ``Physics-informed neural networks (pinns) for fluid mechanics: A review,'' \emph{Acta Mechanica Sinica}, pp. 1--12, 2022.

\bibitem{sahli2020physics}
F.~Sahli~Costabal, Y.~Yang, P.~Perdikaris, D.~E. Hurtado, and E.~Kuhl, ``Physics-informed neural networks for cardiac activation mapping,'' \emph{Frontiers in Physics}, vol.~8, p.~42, 2020.

\bibitem{yang2021b}
L.~Yang, X.~Meng, and G.~E. Karniadakis, ``B-pinns: Bayesian physics-informed neural networks for forward and inverse pde problems with noisy data,'' \emph{Journal of Computational Physics}, vol. 425, p. 109913, 2021.

\bibitem{lu2021physics}
L.~Lu, R.~Pestourie, W.~Yao, Z.~Wang, F.~Verdugo, and S.~G. Johnson, ``Physics-informed neural networks with hard constraints for inverse design,'' \emph{SIAM Journal on Scientific Computing}, vol.~43, no.~6, pp. B1105--B1132, 2021.

\bibitem{shukla2021parallel}
K.~Shukla, A.~D. Jagtap, and G.~E. Karniadakis, ``Parallel physics-informed neural networks via domain decomposition,'' \emph{Journal of Computational Physics}, vol. 447, p. 110683, 2021.

\bibitem{sirignano2018dgm}
J.~Sirignano and K.~Spiliopoulos, ``Dgm: A deep learning algorithm for solving partial differential equations,'' \emph{Journal of computational physics}, vol. 375, pp. 1339--1364, 2018.

\bibitem{grohs2023proof}
P.~Grohs, F.~Hornung, A.~Jentzen, and P.~Von~Wurstemberger, \emph{A proof that artificial neural networks overcome the curse of dimensionality in the numerical approximation of Black--Scholes partial differential equations}.\hskip 1em plus 0.5em minus 0.4em\relax American Mathematical Society, 2023, vol. 284, no. 1410.

\bibitem{darbon2020overcoming}
J.~Darbon, G.~P. Langlois, and T.~Meng, ``Overcoming the curse of dimensionality for some hamilton--jacobi partial differential equations via neural network architectures,'' \emph{Research in the Mathematical Sciences}, vol.~7, no.~3, p.~20, 2020.

\bibitem{darbon2021some}
J.~Darbon and T.~Meng, ``On some neural network architectures that can represent viscosity solutions of certain high dimensional hamilton--jacobi partial differential equations,'' \emph{Journal of Computational Physics}, vol. 425, p. 109907, 2021.

\bibitem{de2022error}
T.~De~Ryck and S.~Mishra, ``Error analysis for physics-informed neural networks (pinns) approximating kolmogorov pdes,'' \emph{Advances in Computational Mathematics}, vol.~48, no.~6, p.~79, 2022.

\bibitem{mishra2023estimates}
S.~Mishra and R.~Molinaro, ``Estimates on the generalization error of physics-informed neural networks for approximating pdes,'' \emph{IMA Journal of Numerical Analysis}, vol.~43, no.~1, pp. 1--43, 2023.

\bibitem{mishra2022estimates}
------, ``Estimates on the generalization error of physics-informed neural networks for approximating a class of inverse problems for pdes,'' \emph{IMA Journal of Numerical Analysis}, vol.~42, no.~2, pp. 981--1022, 2022.

\bibitem{qian2023physics}
Y.~Qian, Y.~Zhang, Y.~Huang, and S.~Dong, ``Physics-informed neural networks for approximating dynamic (hyperbolic) pdes of second order in time: Error analysis and algorithms,'' \emph{Journal of Computational Physics}, vol. 495, p. 112527, 2023.

\bibitem{fang2021high}
Z.~Fang, ``A high-efficient hybrid physics-informed neural networks based on convolutional neural network,'' \emph{IEEE Transactions on Neural Networks and Learning Systems}, vol.~33, no.~10, pp. 5514--5526, 2021.

\bibitem{pang2019fpinns}
G.~Pang, L.~Lu, and G.~E. Karniadakis, ``fpinns: Fractional physics-informed neural networks,'' \emph{SIAM Journal on Scientific Computing}, vol.~41, no.~4, pp. A2603--A2626, 2019.

\bibitem{jiao2021rate}
Y.~Jiao, Y.~Lai, D.~Li, X.~Lu, F.~Wang, Y.~Wang, and J.~Z. Yang, ``A rate of convergence of physics informed neural networks for the linear second order elliptic pdes,'' \emph{arXiv preprint arXiv:2109.01780}, 2021.

\bibitem{oksendal_stochastic_2003a}
B.~{\O}ksendal, \emph{\BIBforeignlanguage{en}{Stochastic {{Differential Equations}}: {{An Introduction}} with {{Applications}}}}, 6th~ed., ser. Universitext.\hskip 1em plus 0.5em minus 0.4em\relax {Berlin Heidelberg}: {Springer-Verlag}, 2003.

\bibitem{borodin_stochastic_2017}
A.~N. Borodin, \emph{Stochastic processes}.\hskip 1em plus 0.5em minus 0.4em\relax Springer, 2017.

\bibitem{lefevre2014survey}
S.~Lef{\`e}vre, D.~Vasquez, and C.~Laugier, ``A survey on motion prediction and risk assessment for intelligent vehicles,'' \emph{ROBOMECH journal}, vol.~1, no.~1, pp. 1--14, 2014.

\bibitem{ferguson2008}
D.~{Ferguson}, M.~{Darms}, C.~{Urmson}, and S.~{Kolski}, ``Detection, prediction, and avoidance of dynamic obstacles in urban environments,'' in \emph{2008 IEEE Intelligent Vehicles Symposium}, 2008, pp. 1149--1154.

\bibitem{khali1996adaptive}
H.~Khali, ``Adaptive output feedback control of nonlinear systems,'' \emph{IEEE Transactions on Automatic Control}, vol.~41, pp. 177--188, 1996.

\bibitem{ono2015chance}
M.~Ono, M.~Pavone, Y.~Kuwata, and J.~Balaram, ``Chance-constrained dynamic programming with application to risk-aware robotic space exploration,'' \emph{Autonomous Robots}, vol.~39, pp. 555--571, 2015.

\bibitem{wang2024myopically}
Z.~Wang, H.~Jing, C.~Kurniawan, A.~Chern, and Y.~Nakahira, ``Myopically verifiable probabilistic certificates for safe control and learning,'' \emph{arXiv preprint arXiv:2404.16883}, 2024.

\bibitem{wachi2024long}
A.~Wachi, W.~Hashimoto, and K.~Hashimoto, ``Long-term safe reinforcement learning with binary feedback,'' in \emph{Proceedings of the AAAI Conference on Artificial Intelligence}, vol.~38, no.~19, 2024, pp. 21\,656--21\,663.

\bibitem{paszke2019pytorch}
A.~Paszke, S.~Gross, F.~Massa, A.~Lerer, J.~Bradbury, G.~Chanan, T.~Killeen, Z.~Lin, N.~Gimelshein, L.~Antiga \emph{et~al.}, ``Pytorch: An imperative style, high-performance deep learning library,'' \emph{Advances in neural information processing systems}, vol.~32, 2019.

\bibitem{abadi2016tensorflow}
M.~Abadi, A.~Agarwal, P.~Barham, E.~Brevdo, Z.~Chen, C.~Citro, G.~S. Corrado, A.~Davis, J.~Dean, M.~Devin \emph{et~al.}, ``Tensorflow: Large-scale machine learning on heterogeneous distributed systems,'' \emph{arXiv preprint arXiv:1603.04467}, 2016.

\bibitem{shin2020convergence}
Y.~Shin, J.~Darbon, and G.~E. Karniadakis, ``On the convergence of physics informed neural networks for linear second-order elliptic and parabolic type pdes,'' \emph{arXiv preprint arXiv:2004.01806}, 2020.

\bibitem{kingma2014adam}
D.~P. Kingma and J.~Ba, ``Adam: A method for stochastic optimization,'' \emph{arXiv preprint arXiv:1412.6980}, 2014.

\bibitem{lu2021deepxde}
L.~Lu, X.~Meng, Z.~Mao, and G.~E. Karniadakis, ``Deepxde: A deep learning library for solving differential equations,'' \emph{SIAM Review}, vol.~63, no.~1, pp. 208--228, 2021.

\bibitem{baydin2018automatic}
A.~G. Baydin, B.~A. Pearlmutter, A.~A. Radul, and J.~M. Siskind, ``Automatic differentiation in machine learning: a survey,'' \emph{Journal of machine learning research}, vol.~18, no. 153, pp. 1--43, 2018.

\bibitem{hoshino2023scalable}
K.~Hoshino, Z.~Wang, and Y.~Nakahira, ``Scalable long-term safety certificate for large-scale systems,'' \emph{IEEE Control Systems Letters}, 2023.

\bibitem{wang2024physics}
Z.~Wang, R.~Keller, X.~Deng, K.~Hoshino, T.~Tanaka, and Y.~Nakahira, ``Physics-informed representation and learning: Control and risk quantification,'' in \emph{Proceedings of the AAAI Conference on Artificial Intelligence}, vol.~38, no.~19, 2024, pp. 21\,699--21\,707.

\bibitem{pham2009continuous}
H.~Pham, \emph{Continuous-time stochastic control and optimization with financial applications}.\hskip 1em plus 0.5em minus 0.4em\relax Springer Science \& Business Media, 2009, vol.~61.

\bibitem{cybenko1989approximation}
G.~Cybenko, ``Approximation by superpositions of a sigmoidal function,'' \emph{Mathematics of control, signals and systems}, vol.~2, no.~4, pp. 303--314, 1989.

\bibitem{peng2020accelerating}
W.~Peng, W.~Zhou, J.~Zhang, and W.~Yao, ``Accelerating physics-informed neural network training with prior dictionaries,'' \emph{arXiv preprint arXiv:2004.08151}, 2020.

\bibitem{gilbarg1977elliptic}
D.~Gilbarg, N.~S. Trudinger, D.~Gilbarg, and N.~Trudinger, \emph{Elliptic partial differential equations of second order}.\hskip 1em plus 0.5em minus 0.4em\relax Springer, 1977, vol. 224, no.~2.

\bibitem{cleveland1981lowess}
W.~S. Cleveland, ``Lowess: A program for smoothing scatterplots by robust locally weighted regression,'' \emph{American Statistician}, vol.~35, no.~1, p.~54, 1981.

\bibitem{wang2022myopically}
Z.~Wang, H.~Jing, C.~Kurniawan, A.~Chern, and Y.~Nakahira, ``Myopically verifiable probabilistic certificate for long-term safety,'' in \emph{2022 American Control Conference (ACC)}.\hskip 1em plus 0.5em minus 0.4em\relax IEEE, 2022, pp. 4894--4900.

\end{thebibliography}

\newpage

\appendix

\section{Proofs}

\subsection{Proof of Theorem~\ref{thm:InvariantProbability_MainTheorem1} and Theorem~\ref{thm:ConvergenceProbability_MainTheorem3}}
\label{sec:Proof_theorem1_3}
The proof of Theorem~\ref{thm:InvariantProbability_MainTheorem1} and Theorem~\ref{thm:ConvergenceProbability_MainTheorem3} is based on the lemmas presented below. These lemmas relate the distributions of various functionals of a diffusion process \(X_t\in\R^n, t\in\R_+\) with deterministic convection-diffusion equations. Let \(X_t\in\R^n, t\in\R_+\) be a solution of the SDE
\begin{equation}\label{eq:Wiener_drift_diffusion}
    d X_{t}=\mu(X_t) d t+\sigma(X_t) d W_{t},
\end{equation}
We assume sufficient regularity in the coefficients of~\eqref{eq:Wiener_drift_diffusion} such that \eqref{eq:Wiener_drift_diffusion} has a unique solution, similar to the assumptions on~\eqref{eq:x_trajectory}.

We first present a variant of the Feynman--Kac Representation~\cite[Theorem 1.3.17]{pham2009continuous}.

\begin{lemma}[Feynman--Kac Representation]\label{thm:Invariant Probability}
Consider the diffusion process~\eqref{eq:Wiener_drift_diffusion} with $X_0=x \in \R^n$. Let $V:\R^n\times \R_+ \to \R$, $\beta: \R^n\times \R_+\to \R$ and \(\psi\colon\R^n\to\R\)
be some given functions. The solution of the Cauchy problem 
\begin{align}
\begin{cases}
    \frac{\partial \lemmaF}{\partial T} 
    =  
    \textstyle\frac{1}{2}\tr(\sigma\sigma^\intercal\Hess{\lemmaF}) + \gL_\mu\lemmaF
    - V\lemmaF + \beta,
 &T>0,
    \\
    \lemmaF(x,0) = \psi(x) 
\end{cases}
\end{align}
is given by
\begin{align}
\begin{split}
\lemmaF(x,T) &= \mE_x\Biggl[
\textstyle
e^{-\int_0^T V(X_t,T-t)\, dt}
\psi(X_T)\\
&\left. + 
\int_0^T
\textstyle e^{-\int_0^\tau V(X_t,T-t)\, dt}
\beta(X_\tau,T-\tau)\, d\tau \right].
\end{split}
\end{align}
\end{lemma}

While the original statement given in~\cite[Theorem 1.3.17]{pham2009continuous} presents \(\lemmaF\) as a function of \(t\) that moves backward in time, we present \(\lemmaF\) as a function of \(T\) that moves forward in time.  Even though both approaches express the time derivative with opposite signs, both yield the same probability. The representation in Lemma~\ref{thm:Invariant Probability} can be used to derive the following result. 

\begin{lemma}\label{lem:Lemma1}
Let \(\gM\subset\R^n\) be a domain.  Consider the diffusion process \eqref{eq:Wiener_drift_diffusion} with the initial condition \(X_0=x\in\R^n\).
Then the function \(\lemmaF:\R^n\times [0,\infty)\to \R\) defined by
\begin{align}
U(x,T):=\mP_x(X_t\in\gM, \forall t\in [0,T])
\end{align}
is the solution to the convection-diffusion equation
\begin{align}
\label{eq:Lemma1CDE}
    \begin{cases}
     {\partial \lemmaF\over\partial T} = {1\over 2}\tr(\sigma\sigma^\intercal\Hess\lemmaF) + \gL_\mu \lemmaF,
     & x\in\gM, T>0,\\
     \lemmaF(x,T) = 0,&x\notin\gM, T>0,\\
     \lemmaF(x,0) = \mathbb{1}_{\gM}(x),& x\in\R^n.
    \end{cases}
\end{align}
\end{lemma}


\begin{proof}
\rev{
Note that Lemma~\ref{thm:Invariant Probability} holds for any given functions $V:\R^n\times \R_+ \to \R$, $\beta: \R^n\times \R_+\to \R$ and \(\psi\colon\R^n\to\R\). We apply Lemma~\ref{thm:Invariant Probability} with
\begin{align*}
    \beta(x,t) \equiv 0,\, \psi(x) = \mathbb{1}_{\gM}(x),\,
    V(x,t) = \gamma\mathbb{1}_{\gM^c}(x),
\end{align*}
for any $\gamma >0$, where $\gM^c$ is the complement of $\gM$.
}
Then we have
\begin{align}
\label{eq:e2p_0}
    \lemmaF(x,T)
& = \mE_x\left[ e^{-\int^T_0 \gamma \mathbb{1}_{\gM^c}(X_t)\,dt}\psi(X_T)\right].
\end{align}
and
\begin{align}
\label{eq:eq_1}
    \begin{cases}
    \frac{\partial \lemmaF}{\partial T} = \frac{1}{2}\tr(\sigma\sigma^\intercal\Hess{\lemmaF}) + \gL_\mu\lemmaF,&x\in\gM, T>0,\\
    \frac{\partial \lemmaF}{\partial T} = \frac{1}{2}\tr(\sigma\sigma^\intercal\Hess{\lemmaF}) + \gL_\mu\lemmaF -\gamma\lemmaF,&x\notin \gM, T>0,\\
    \lemmaF(x,0) = \mathbb{1}_{\gM}(x),& x\in\R^n.
\end{cases}
\end{align}

Now, we take the limit of \(\gamma\to\infty\).  Since
\begin{align*}
    \lim_{\gamma \rightarrow \infty} e^{-\int^T_0 \gamma \mathbb{1}_{\gM^c}(X_t)\,dt } = 
\begin{cases}
0, & \text{if } X_t \notin \gM , \exists t\in[0,T],\\
1, & \text{otherwise}.
\end{cases},
\end{align*} \eqref{eq:e2p_0} becomes
\begin{align}\label{eq:Lemma1ProofFN-Solution}
    U(x,T)&= \mE_x[\mathbb{1}_{\{ X_t \in \gM , \forall t\in[0,T] \}}] \\
    & = \mP_x(X_t \in \gM , \forall t\in[0,T]).
\end{align}
Meanwhile, under the limit of \(\gamma\to\infty\), the part of \eqref{eq:eq_1} with \(x\notin \gM\) and \(T>0\)
reduces to the algebraic condition
\begin{align}\label{eq:Lemma1ProofCon}
    \lemmaF(x,T) = 0,\quad x\notin \gM, T>0.
\end{align}
Combining~\eqref{eq:eq_1}, \eqref{eq:Lemma1ProofFN-Solution}, and~\eqref{eq:Lemma1ProofCon} and gives Lemma~\ref{lem:Lemma1}.
\end{proof}

Using Lemma~\ref{lem:Lemma1}, we can prove Theorem~\ref{thm:InvariantProbability_MainTheorem1} and Theorem~\ref{thm:ConvergenceProbability_MainTheorem3} as follows.

\begin{proof}(Theorem~\ref{thm:InvariantProbability_MainTheorem1})
Consider the augmented space space $Z_t = [ \phi(X_t), X_t^\intercal]^\intercal \in \R^{n+1}$. The stochastic process \(Z_t\) is the solution to~\eqref{eq:MainTheorem_z} with parameters $\ave$ and $\var$ defined in~\eqref{eq:mu_sigma_prime} with the initial state
\begin{align}
    Z_0 = z = [\phi(x),x^\intercal]^\intercal.
\end{align}
Let us define 
\begin{align}
\label{eq:set_thm1} 
    \gM = \{ z \in \R^{n+1} : z[1] \geq \ell \} .
\end{align}
From Lemma~\ref{lem:Lemma1},
\begin{align*}
    F(x,T;\ell) = \mP_x( z_t \in \gM , \forall t \in [ 0, T] )
\end{align*}
is the solution to the convection-diffusion equation \begin{align}
\label{eq:lem1_pde}
\begin{cases}
{\partial \FwdInv \over\partial T} = {1\over 2}\operatorname{tr}(\zeta\zeta^\intercal\operatorname{Hess}\FwdInv) + \gL_\rho \FwdInv,& z\in \gM, T>0,\\
\FwdInv(z,T) = 0,& z\notin\gM,T>0,\\
\FwdInv(z,0) = \mathbb{1}_{\gM}(z),&z\in\R^{n+1}.
\end{cases}
\end{align}
To obtain  \eqref{eq:CauchyProblem}, define \(D = \zeta\zeta^\intercal\) and apply the vector identity\footnote{The identity is a direct consequence of the Leibniz rule \(\partial_i(D_{ij}\partial_j F) = (\partial_iD_{ij})\partial_j F + D_{ij}\partial_i\partial_j F.\)}
\begin{align}
\label{eq:VectorIdentity}
    \nabla\cdot (D\grad\FwdInv) = \gL_{\nabla\cdot D}\FwdInv + \tr(D\Hess\FwdInv).
\end{align}
\end{proof}
\begin{proof}(Theorem~\ref{thm:ConvergenceProbability_MainTheorem3})
Consider the augmented space of $Z_t = [ \phi(X_t), X_t^\intercal]^\intercal \in \R^{n+1}$ in~\eqref{eq:augumented_z}. The stochastic process $Z_t$ is a solution of~\eqref{eq:MainTheorem_z} with parameters $\ave$ and $\var$ defined in~\eqref{eq:mu_sigma_prime} with the initial state  
\begin{align}
\label{eq:thm3-initial}
    Z_0 = z  = [ \phi(x), x^\intercal ]^\intercal.
\end{align}
Let us define 
\begin{align}
\label{eq:set_thm3} 
    \gM = \{ z \in \R^{n+1} : z[1] < \ell \} .
\end{align}
The CDF of $\maxf_x(T)$ is given by
\begin{align}
    \mP(\maxf_x(T) <\ell ) &= \mP_x(\forall t\in[0,T], \phi(X_t)<\ell)\\
    &=\mP_x(\forall t\in[0,T],Z_t\in\gM),
\end{align}
which, by Lemma~\ref{lem:Lemma1}, is the solution to the convection-diffusion equation \eqref{eq:Lemma1CDE}, yielding \eqref{eq:pde_thm3} after the application of identity \eqref{eq:VectorIdentity}.
\end{proof}

\subsection{Proof of Theorem~\ref{thm:InvariantProbability_MainTheorem2} and Theorem~\ref{thm:ConvergenceProbability_MainTheorem4}}
\label{sec:Proof_theorem2_4}


\rev{
We first present
Lemma~\ref{lem:EscapeTime} and Lemma~\ref{lem:EscapeTimeLemma3}, which link the 
escape time of a diffusion process to PDE solutions via a Feynman--Kac formulation, 
and then use these lemmas to prove 
Theorem~\ref{thm:InvariantProbability_MainTheorem2} and 
Theorem~\ref{thm:ConvergenceProbability_MainTheorem4}.
}




\begin{lemma}\label{lem:EscapeTime}
Consider the diffusion process \eqref{eq:Wiener_drift_diffusion} with the initial point \(X_0=x\). Define the escape time as~\eqref{eq:EscapeTime}. 
Let \(\psi(x), V(x)\) be continuous functions and \(V\) be non-negative. If $\lemmaF: \R^{n} \to \R$ is the bounded solution to the boundary value problem
\begin{align}\label{eq:Col_boundary_condition}
    \begin{cases}
    {1\over 2}\tr(\sigma\sigma^\intercal \Hess{\lemmaF}) + \gL_\mu\lemmaF - V\lemmaF =0, & x\in\gM,\\
    \lemmaF(x) = \psi(x),& x\notin \gM,\\
    \end{cases}
\end{align}
then
\begin{align}
\label{eq:TimeIndependentSolution}
    \lemmaF(x) = \mE_x\left[\psi(X_{\lemmaT_{\gM}})e^{-\int_0^{\lemmaT_{\gM}}V(X_{s})\, ds}\right].
\end{align}
\end{lemma}


\begin{proof}
We first define a mapping $\eta: \R \rightarrow \R$ by 
 \begin{equation}\label{eq:67}
     \eta(s) := \lemmaF(X_s)e^{-\int_0^s V(X_v)\,dv}.
 \end{equation}
It satisfies
\begin{align}
\label{eq:noise_eta0}
    \eta(q)  &= \int_0^q d\eta(s) + \eta(0)\\
    \nonumber
    &=\int_0^q e^{-\int_0^s V(X_v)\, dv}\bigl[
    -V(X_s)\lemmaF(X_s) + \gL_\mu \lemmaF(X_s)\\
    \label{eq:noise_eta1}
    &\quad+\textstyle{1\over 2}\tr \left(\sigma(X_s)\sigma^\intercal(X_s)\Hess \lemmaF(X_s)\right)
    \bigr]ds \\
    \nonumber
    &\quad + \int_0^q e^{-\int_0^s V(X_v)\, dv}\gL_\sigma U(X_s)\, dW_s + \eta(0) \\
    &=\int_0^q
    e^{-\int_0^s V(X_v)\, dv}\gL_\sigma U(X_s)\, dW_s + \eta(0) ,
    \label{eq:noise_eta}
\end{align}
where \eqref{eq:noise_eta1} is from It\^o's Lemma; \eqref{eq:noise_eta} is from \eqref{eq:Col_boundary_condition} with $x \in \gM$. Thus, its expectation satisfies
\begin{align}
 \label{eq:eta_F1}
    \mE_x[\eta(q)] & = \mE_x[\eta(0)] \\
    \label{eq:eta_F2}
    &=  \mE_x[\lemmaF(X_0)] \\
    &= \lemmaF(x).
     \label{eq:eta_F}
\end{align}
where \eqref{eq:eta_F1} holds because the right hand side of \eqref{eq:noise_eta} has zero mean; \eqref{eq:eta_F2} is due to \eqref{eq:67}; and \eqref{eq:eta_F} is from the assumption that $X_0 = x$. 

Next, we set
\begin{align}
\label{eq:q-value}
   q = H_{\gM} .
\end{align}
As \eqref{eq:q-value} implies \(X_{q}\notin\gM\), condition \eqref{eq:Col_boundary_condition} with $x \notin \gM$ yields  
\begin{align}
\label{eq:UGetsBdyValue1}
    \lemmaF(X_q) = \psi(X_q). 
\end{align}
Finally, we have 
\begin{align}
\label{eq:thm4_ucondition1}
    \lemmaF(x)&= \mE_x[\eta(q)]\\
    \label{eq:thm4_ucondition2}
    &=\mE_x\left[\lemmaF(X_q)e^{-\int_0^q V(X_v)\,dv} \right]\\
    \label{eq:thm4_ucondition3}
    &= \mE_x\left[\psi(X_q)e^{-\int_0^q V(X_v)\,dv }\right]\\
    \label{eq:thm4_ucondition4}
     &= \mE_x\left[\psi(X_{H_{\gM}})e^{-\int_0^{H_{\gM}} V(X_v)\,dv }\right],
\end{align}
where \eqref{eq:thm4_ucondition1} is due to \eqref{eq:eta_F}; \eqref{eq:thm4_ucondition2} is due to \eqref{eq:67}; \eqref{eq:thm4_ucondition3} is due to \eqref{eq:UGetsBdyValue1}, and \eqref{eq:thm4_ucondition4} is due to \eqref{eq:q-value}. 

\end{proof}

\begin{lemma}\label{lem:EscapeTimeLemma3}
Consider the diffusion process \eqref{eq:Wiener_drift_diffusion} with the initial condition \(X_0=x\in\R^n\).
Define the escape time,
\begin{align}\label{eq:EscapeTime}
    \lemmaT_{\gM}:=\inf\{t\in \R_+: X_t\notin\gM\}.
\end{align}
The CDF of the escape time
\begin{align}
    \lemmaF(x,t) = \mP_x(\lemmaT_\gM\leq t),\quad t>0,
\end{align}
is the solution to
\begin{align}\label{eq:EscapeTimeLemma3_boundary_condition}
    \begin{cases}
    {\partial \lemmaF \over \partial t}(x,t) = {1\over 2}\tr(\sigma\sigma^\intercal  \Hess{\lemmaF}) + \gL_\mu\lemmaF, & x\in\gM,t>0,\\
    \lemmaF(x,t) = 1, & x\notin\gM,t>0,\\
    \lemmaF(x,0) = \mathbb{1}_{\gM^c}(x), & x\in\R^n.\\
    \end{cases}
\end{align}
\end{lemma}


\begin{proof}
Let \(\gamma\in\Bbb C\) be a spectral parameter with \(\operatorname{Re}(\gamma)> 0\). For each fixed \(\gamma\), let $\hat{\lemmaF}(\cdot,\gamma): \R^n \to \Bbb C$ be the solution of
\begin{align}\label{eq:InvariantProbability_FeynmaKac_PDE_complex}
\begin{cases}
         \frac{1}{2}\tr(\sigma\sigma^\intercal\Hess{\hat{\lemmaF}}) + \gL_\mu\hat{\lemmaF} -\gamma \hat{V}\hat{\lemmaF}=0, & x\in\gM,\\
         \hat{\lemmaF}(x,\gamma) = \hat{\psi}(x,\gamma), & x\notin\gM .\\
\end{cases}
\end{align}
According to Lemma~\ref{lem:EscapeTime}, \(\hat{\lemmaF}(x,\gamma)\) is given by
\begin{align}\label{eq:78}
    \hat \lemmaF(x,\gamma) & = \mE_x\left[\hat{\psi}(X_{H_{\gM}})e^{-\gamma\int_0^{H_{\gM}}\hat{V}(X_s)\, ds}\right],
\end{align}
where \(X_s\) is the diffusion process in~\eqref{eq:Wiener_drift_diffusion}. Now, take
\begin{align}\label{eq:SpecialPsiV}
    \hat{\psi}(x,\gamma) = 1/\gamma,\quad \hat{V}(x) = 1,
\end{align}
in~\eqref{eq:InvariantProbability_FeynmaKac_PDE_complex} and~\eqref{eq:78}.
Then,~\eqref{eq:InvariantProbability_FeynmaKac_PDE_complex} becomes
\begin{align}
\label{eq:LaplaceDomainPDE}
\begin{cases}
         \frac{1}{2}\tr(\sigma\sigma^\intercal\Hess{\hat{\lemmaF}})+ \gL_\mu\hat{\lemmaF}-\gamma\hat{\lemmaF}=0, & x\in\gM,\\
         \hat{\lemmaF}(x,\gamma) = 1/ \gamma, & x\notin\gM,\\
\end{cases}
\end{align}
and \eqref{eq:78} becomes
\begin{align}
\nonumber
    \hat \lemmaF(x, \gamma) & = \mE_x\left[\tfrac{1}{\gamma}e^{-\gamma H_{\gM}}\right]\\
    \nonumber
    & = \int_0^\infty {1\over\gamma} e^{-\gamma t}p_{H_{\gM}\mid X_0}(t\mid x)\, dt\\
    \label{eq:LaplaceCDF3}
    &=\int_0^\infty e^{-\gamma t}\mP_x(H_{\gM}\leq t)\, dt
\end{align}
where integration by parts is used in \eqref{eq:LaplaceCDF3}.
Here, 
\(
    p_{H_{\gM}|X_0}(t|x) = {d \over dt}\mP_x(H_{\gM}\leq t)
\)
denotes the probability density function of \(H_\gM\) conditioned on \(X_0 = x\). 

Now, for fixed \(x\), let \(\lemmaF(x,t)\)  be the inverse Laplace transformation of \(\hat{\lemmaF}(x,\gamma)\). In other words,
\begin{align}\label{eq:Laplace}
    \hat \lemmaF(x,\gamma) & = \int_0^\infty \lemmaF(x,t)e^{-\gamma t}\, dt.
\end{align}
On one hand, comparing~\eqref{eq:LaplaceCDF3} and~\eqref{eq:Laplace}  gives
\begin{align}
    \lemmaF(x,t) = \mP_x\left(H_{\gM} \leq t\right).
\end{align}
On the other hand,
taking the inverse Laplace transformation of the PDE~\eqref{eq:LaplaceDomainPDE} for \(\hat \lemmaF\) yields
the PDE satisfied by \(\lemmaF\):
\begin{align}
\begin{cases}
         {\partial \lemmaF \over \partial t}(x,t)=\frac{1}{2}\tr(\sigma\sigma^\intercal\Hess{\lemmaF}) + \gL_\mu\lemmaF, & x\in\gM,t>0,\\
         \lemmaF(x,t) = 1, & x\notin\gM,t>0,\\
        \lemmaF(x,0) = \mathbb{1}_{\gM^c}(x), &  x\in\R^n.
\end{cases}
\end{align}
\end{proof}


Now, we are ready to prove Theorem~\ref{thm:InvariantProbability_MainTheorem2} and Theorem~\ref{thm:ConvergenceProbability_MainTheorem4}.
\begin{proof} (Theorem~\ref{thm:InvariantProbability_MainTheorem2})
Let the augmented state \(Z_t\) be the solution of~\eqref{eq:MainTheorem_z} with the initial state 
    \(Z_0 = z = [\phi(x), x^\intercal]^\intercal\in\R^{n+1}.\)
Let 
\begin{align}
\label{eq:set_thm2}
    \gM = \{z\in\R^{n+1}: z[1]\geq \ell\}.
\end{align}
Thus, we have \(\exit_x(\ell) = H_{\gM}\), for \(\FwdInvExit\) be the solution of the convection-diffusion equation with \(D = \zeta\zeta^\intercal\)
\begin{align*}\
\begin{cases}
    \frac{\partial \FwdInvExit}{\partial t} =  \frac{1}{2}\grad\cdot(D\grad{\FwdInvExit}) + \gL_{\rho-{1\over 2}\grad{ \cdot} D}\FwdInvExit
    ,& z\in\gM, t>0,\\
    \FwdInvExit(z,t) = 1,  &  z\notin\gM, t>0.\\
    \FwdInvExit(z,0) = \mathbb{1}_{\gM^c}(z),&z\in\R^{n+1},
\end{cases}
\end{align*}
which is equivalent to~\eqref{eq:MainTheorem2}. 
From Lemma~\ref{lem:EscapeTimeLemma3}, we have 
\(
   \FwdInvExit(z,t;l)  = \mP(\exit_{x}(l) \leq t).
\)
\end{proof}

\begin{proof}(Theorem~\ref{thm:ConvergenceProbability_MainTheorem4})
Let the augmented state space \(Z_t\) be the solution of~\eqref{eq:MainTheorem_z} with the initial state 
    \(Z_0 = z = [\phi(x), x^\intercal]^\intercal\in\R^{n+1}.\)
Let 
\begin{equation}
\label{eq:set_thm4}
    \gM = \{z\in\R^{n+1}: z[1] < \ell\}.
\end{equation}
We have \(\entrance_x(\ell) = H_{\gM}\) for \(\FwdConvExit\) be the solution of the convection diffusion equation with \(D = \zeta\zeta^\intercal\)
\begin{align*}
\begin{cases}
    \frac{\partial \FwdConvExit}{\partial t}(z,t) =  \frac{1}{2}\grad\cdot(D\grad{\FwdConvExit}) + \gL_{\ave-{1\over 2} \grad{\cdot D}}{\FwdConvExit}
    ,& z\in\gM,t>0,\\
    \FwdConvExit(x,t) = 1,  & z\notin\gM,t>0,\\
    \FwdConvExit(x,0) = \mathbb{1}_{\gM^c}(z),&z\in\R^{n+1},
\end{cases}
\end{align*}
which is equivalent to~\eqref{eq:PDE_D}.
From Lemma~\ref{lem:EscapeTimeLemma3}, we have 
\(
   \FwdConvExit(z,T;l)  = \mP(\entrance_{x}(l) \leq T).
\)
\end{proof}

\subsection{Proof of Theorems for Physics-informed Learning}
\label{apx:proof_pinn}

\begin{proof}(Corollary~\ref{cor:NN_worst_case_bound})
We know that $u$ is the solution to the PDE of interest. 
We can construct $\Bar{u}$ such that
\begin{equation}
    \bar{u} = \begin{cases}
    u, \quad (x,T) \in \D_s \\
    u + \frac{d}{d_{\text{max}}}(M+\delta), \quad (x,T) \in \Omega \times \tau \backslash \D_s
    \end{cases}
\end{equation}
where $d$ characterizes the distance between $(x,t) \in \Omega \backslash \D_s$ and set $\D_s$, $d_{\text{max}} = \max_{(x,t) \in \Omega \backslash \D_s} d$ is the maximum distance, and $0 < \delta < \min\{\delta_1, \delta_2\}$ is a constant. Then we have
\begin{equation}
    \sup_{(x, T) \in \Omega \times \tau} \left|\bar{u}(x, t)-u(x, t)\right| = M + \delta
\end{equation}
Since we assume the neural network has sufficient representation power, by universal approximation theorem~\cite{cybenko1989approximation}, for $\delta$ given above, there exist $\Bar{\boldsymbol\theta}$ such that
\begin{equation}
    \sup _{(x, T) \in \Omega \times \tau}\left|F_{\Bar{\boldsymbol\theta}}(x, t)-\bar{u}(x, t)\right| \leq \delta.
\end{equation}
Then we have $\Bar{\boldsymbol\theta}$ satisfies 
\begin{equation}
\begin{aligned}
\sup _{(x, T) \in  \Sigma_s}|{F}_{\Bar{\boldsymbol\theta}}(x, T)-\Tilde{u}(x, T)|<\delta_1, \\
\sup _{(x, T) \in \D_s}|{F}_{\Bar{\boldsymbol\theta}}(x, T)-{u}(x, T) |<\delta_2,
\end{aligned}
\end{equation}
and 
\begin{equation}
    \sup _{(x, T) \in \Omega \times \tau}\left|F_{\Bar{\boldsymbol\theta}}(x, t)-u(x, t)\right| \geq M.
\end{equation}

\end{proof}

The proof of Theorem~\ref{thm:full_pde_constraint} is based on~\cite{peng2020accelerating, gilbarg1977elliptic} by adapting the results on elliptic PDEs to parabolic PDEs. We first give some supporting definitions and lemmas.
We define the second order parabolic operator $L$ w.r.t. $u$ as follows.
\begin{equation}
\label{eq:parabolic_operator}
\begin{aligned}
     L[u]:= & -\frac{\partial u}{\partial T} + \sum_{i, j} a_{i, j}(x, T) \frac{\partial^2 u}{\partial x_i x_j} \\
    & \qquad +\sum_i b_i(x, T) \frac{\partial u}{\partial x_i}+c(x, T) u.
\end{aligned}
\end{equation}
Let $A(x):=\left[a_{i, j}(x, T)\right]_{i, j} \in \mathbb{R}^{d \times d}$ and $b(x, T):=\left[b_i(x)\right]_i \in \mathbb{R}^d$. We consider uniform parabolic operators, where $A$ is positive definite, with the smallest eigenvalue $\lambda > 0$. 

Let $\Omega \times \tau \in \mathbb{R}^{d+1}$ be a bounded domain of interest, where $\Sigma$ is the boundary of $\Omega \times \tau$. We consider the following partial differential equation for the function $u(\cdot)$.
\begin{equation}
\label{eq:parabolic_PDE}
\begin{aligned}
    L[u](x, T) & = q(x, T), \quad (x, T) \in \Omega \times \tau \\
    u(x,T) & = \tilde{u}(x, T), \quad (x, T) \in \Sigma
\end{aligned}
\end{equation}
We know that the risk probability PDEs (18)-(21) are parabolic PDEs and can be written in the form of~\eqref{eq:parabolic_PDE}.
Specifically, in our case we have $c(x,T) \equiv 0$
and $A = \frac{1}{2}\sigma^2 I$ which gives $\lambda = \frac{1}{2}\sigma^2$.

For any space $\Omega \in \mathbb{R}^d$, for any function $f: \mathbb{R}^d \rightarrow \mathbb{R}$, we define the $L_1$ norm of the function $f$ on $\Omega$ to be
\begin{equation}
    \|f\|_{L_1 (\Omega)} := \int_\Omega f(\mathbf{X}) d\mathbf{X}.
\end{equation}
With this definition, we know that 
\begin{equation}   \mathbb{E}_{\mathbf{X}}\left[f\right] =  \|f\|_{L_1 (\Omega)} / |\Omega|,
\end{equation}
where $\mathbf{X}$ is uniformly sampled from $\Omega$, and $|\Omega|$ denote the size of the space $\Omega$.

\begin{corollary}
\label{cor:maximum_principle}
\textbf{Weak maximum principle}. Suppose that $\D = \Omega \times \tau$ is bounded, $L$ is uniformly parabolic with $c \leq 0$ and that $u \in C^0(\bar{\D}) \cap C^2(\D)$ satisfies $Lu \geq 0$ in $\D$, and $M = \max_{\Bar{\D}} u \geq 0$. Then
\begin{equation}
    \max_{\Bar{\D}} u = \max_{\Sigma} u.
\end{equation}
\end{corollary}

\begin{corollary}
\label{cor:comparison_principle}
\textbf{Comparison principle.} Suppose that $\D$ is bounded and $L$ is uniformly parabolic.
If $u, v \in C^0(\bar{\D}) \cap C^2(\D)$ satisfies $Lu \leq Lv$ in $\D$ and $u \geq v$ on $\Sigma$, then $u \geq v$ on $\D$.
\end{corollary}

\begin{theorem}
\label{thm:parabolic_PDE_bound}
Let $Lu = q$ in a bounded domain $\D$, where $L$ is parabolic, $c \leq 0$ and $u \in C^0(\bar{\D}) \cap C^2(\D)$. Then
\begin{equation}
    \sup_\D |u| \leq \sup_{\Sigma} |u| + C \sup_\D \frac{|q|}{\lambda},
\end{equation}
where $C$ is a constant depending only on diam $\D$ and $\beta = \sup |b|/\lambda$.
\end{theorem}

\begin{proof}
(Theorem~\ref{thm:parabolic_PDE_bound})
Let $\D$ lie in the slab $0 < x_1 < d$. Without loss of generality, we assume $\lambda_1 \geq \lambda > 0$. Set $L_0 = -\frac{\partial}{\partial t} + a^{ij}\frac{\partial^2}{\partial x_i x_j} + b^i\frac{\partial}{\partial x_i}$. For $\alpha \geq \beta + 1$, we have
\begin{equation}
    L_0 e^{\alpha x_1} = (\alpha^2 a^{11} + \alpha b^1) e^{\alpha x_1} \geq \lambda (\alpha^2 - \alpha \beta) e^{\alpha x_1} \geq \lambda
\end{equation}
Consider the case when $Lu \geq q$. Let
\begin{equation}
    v = \sup_{\Sigma} u^+ + (e^{\alpha d} - e^{\alpha x_1})\sup_\D \frac{|q^-|}{\lambda},
\end{equation}
where $u^+ = \max (u, 0)$ and $q^- = \min(q, 0)$.
Then, since $Lv = L_0 v + cv \leq -\lambda \sup_{\Sigma} (|q^-|/\lambda)$ by maximum principle (Corollary~\ref{cor:maximum_principle}), we have
\begin{equation}
    L(v-u) \leq -\lambda (\sup_{\D} \frac{|q^-|}{\lambda} + \frac{q}{\lambda}) \leq 0 \text{ in } \D,
\end{equation}
and $v-u \geq 0$ on $\Sigma$. Hence, from comparison principle (Corollary~\ref{cor:comparison_principle}) we have
\begin{equation}
    \sup_{\D} u \leq \sup_{\D} v \leq \sup_{\Sigma} |u| + C \sup_\D \frac{|q|}{\lambda}
\end{equation}
with $C = e^{\alpha d} - 1$.
Consider $Lu \leq q$, we can get similar results with flipped signs. Combine both cases we have for $Lu=q$
\begin{equation}
    \sup_\D |u| \leq \sup_{\Sigma} |u| + C \sup_\D \frac{|q|}{\lambda}.
\end{equation}
\end{proof}

\begin{theorem}
\label{thm:PINN_bound}
Suppose that $\D \in \mathbb{R}^{d+1}$ is a bounded domain, $L$ is the parabolic operator defined in~\eqref{eq:parabolic_operator} and $u \in C^0(\bar{\D}) \cap C^2(\D)$ is a solution to the risk probability PDE. If the PINN ${F}_{\boldsymbol\theta}$ satisfies the following conditions:
\begin{enumerate}
    \item $\sup _{(x, T) \in \Sigma}| {F}_{\boldsymbol\theta}(x, T)-\Tilde{u}(x, T)|<\delta_1$;
    \item $\sup _{(x, T) \in \D}|W_{{F}_{\boldsymbol\theta}}(x,T)|<\delta_2$;
    \item ${F}_{\boldsymbol\theta} \in C^0(\bar{\D}) \cap C^2(\D)$,
\end{enumerate}
Then the error of $F_{\boldsymbol\theta}$ over $\Omega$ is bounded by
\begin{equation}
\label{eq:PINN_bound}
    \sup _{(x, T) \in \D}\left|{F}_{\boldsymbol\theta}(x, T)-u(x, t)\right| \leq {\delta_1}+C \frac{{\delta_2}}{\lambda}
\end{equation}
where $C$ is a constant depending on $\D$ and $L$.
\end{theorem}

\begin{proof}
(Theorem~\ref{thm:PINN_bound})
Denote $h_1 = L\left[{F}_{\boldsymbol\theta}\right] - q$, and $h_2 = {F}_{\boldsymbol\theta} - u$. Since ${F}_{\boldsymbol\theta}$ and $u$ both fall in $C^0(\bar{\D}) \cap C^2(\D)$, from Theorem~\ref{thm:parabolic_PDE_bound} we have
\begin{equation}
\begin{aligned}
    \sup _{\D}\left|h_2(x, t)\right| & \leq \sup_{\Sigma} |h_2(x, t)| + C \sup_\D \frac{|h_1(x, t)|}{\lambda} \\
    & \leq {\delta_1}+C \frac{{\delta_2}}{\lambda} 
\end{aligned}
\end{equation}
which gives~\eqref{eq:PINN_bound}.
\end{proof}

\begin{lemma}
\label{lem:Lipshitz_bound}
Let $\D \subset \mathbb{R}^{d+1}$ be a domain. Define the regularity of $\D$ as

\begin{equation}
    R_{\D}:=\inf _{(x,T) \in \D, r>0} \frac{|B(x, T,  r) \cap \D|}{\min \left\{|\D|, \frac{\pi^{(d+1) / 2} r^{d+1}}{\Gamma((d+1) / 2+1)}\right\}},
\end{equation}
where $B(x, T, r):=\left\{y \in \mathbb{R}^{d+1} \mid\|y-(x,T)\| \leq r\right\}$ and $|S|$ is the Lebesgue measure of a set $S$. Suppose that $\D$ is bounded and $R_{\D}>0$. Let $q \in C^{0}(\bar{\D})$ be an $l$-Lipschitz continuous function on $\bar{\D}$. Then

\begin{equation}
\label{eq:Lipshitz_bound}
\begin{aligned}
    \sup _{\D}|q| \leq & \max \biggl\{\frac{2\|q\|_{L_{1}(\bar{\D})}}{R_{\D}|\D|}, \\
& \quad 2 l \cdot\left(\frac{\|q\|_{L_{1}(\bar{\D})} \cdot \Gamma((d+1) / 2+1)}{l R_{\D} \cdot \pi^{(d+1) / 2}}\right)^{\frac{1}{d+2}}\biggl\}.  
\end{aligned}
\end{equation}
\end{lemma}

\begin{proof}
(Lemma~\ref{lem:Lipshitz_bound})
According to the definition of $l$-Lipschitz continuity, we have
\begin{equation}
    l\|(x,T)-(\bar{x}, \bar{T})\|_2 \geq|q(x,T)-q(\bar{x}, \bar{T})|, \quad \forall (x,T), (\bar{x}, \bar{T}) \in \bar{\D},
\end{equation}
which follows
\begin{equation}
\label{eq:L1_bound}
\begin{aligned}
    \|q\|_{L_{1}(\bar{\D})} \geq & \int_{\D^{+}}|q(x, T)| d x dT \\
    \geq & \int_{\D^{+}}|f(\bar{x}, \bar{T})|-l\|(x,T)-(\bar{x}, \bar{T})\| d x dT,
\end{aligned}
\end{equation}
where $\D^{+}:=\{(x,T) \in \bar{\D} \mid |q(\bar{x},\bar{T})| -l\|(x,T)-(\bar{x}, \bar{T})\| \geq 0\}$. Without loss of generality, we assume that $(\bar{x}, \bar{T}) \in \arg \max _{\bar{\D}}|q|$ and $q(\bar{x}, \bar{T})>0$. Denote that
\begin{equation}
    B_{1}:=B\left(\bar{x}, \bar{T}, \frac{q(\bar{x}, \bar{T})}{2 l}\right) \cap \D.
\end{equation}
It obvious that $B_{1} \subset \D^{+}$. Note that the Lebesgue measure of a hypersphere in $\mathbb{R}^{d+1}$ with radius $r$ is $\pi^{(d+1) / 2} r^{d+1} / \Gamma((d+1) / 2+1)$. Then~\eqref{eq:L1_bound} becomes
\begin{equation}
\begin{aligned}
    & \|q\|_{L_{1}(\bar{\D})} \geq \int_{B_{1}} q(\bar{x})-l\|(x,T)-(\bar{x}, \bar{T})\| d x dT\\
    \geq & \left|B_{1}\right| \cdot \frac{q(\bar{x}, \bar{T})}{2} \\
    \geq & \frac{q(\bar{x}, \bar{T})}{2} \cdot R_{\D} \cdot \min \left\{|\D|, \frac{\pi^{(d+1) / 2} q(\bar{x}, \bar{T})^{d+1}}{2^{d+1} l^{d+1} \Gamma((d+1) / 2+1)}\right\} \\
    = & \sup_{D} |q| \cdot \frac{R_{\D}}{2} \cdot \min \left\{|\D|, \frac{\pi^{(d+1) / 2} q(\bar{x}, \bar{T})^{d+1}}{2^{d+1} l^{d+1} \Gamma((d+1) / 2+1)}\right\},
\end{aligned}
\end{equation}
which leads to~\eqref{eq:Lipshitz_bound}.
\end{proof} 

Now we are ready to prove Theorem~\ref{thm:full_pde_constraint}.

\begin{proof}
(Theorem~\ref{thm:full_pde_constraint})
From condition 1, $\mathbf{Y}$ is uniformly sampled from $\Sigma$. From condition 2, $\mathbf{X}$ is uniformly sampled from $\D$. Then we have
\begin{equation}
    \mathbb{E}_{\mathbf{Y}}\left[|{F}_{\boldsymbol\theta}(\mathbf{Y})-\Tilde{u}(\mathbf{Y})|\right] = 
    \| {F}_{\boldsymbol\theta}(x, T)-\Tilde{u}(x, T)\|_{L_1(\Sigma)}/|\Sigma|,
\end{equation}
\begin{equation}
    \mathbb{E}_{\mathbf{X}}\left[|W_{{F}_{\boldsymbol\theta}}(\mathbf{X})|\right] = 
    \| W_{{F}_{\boldsymbol\theta}}(x, T)\|_{L_1(\D)}/|\D|.
\end{equation}
From condition 1 and 2 we know that
\begin{equation}
     \| {F}_{\boldsymbol\theta}(x, T)-\Tilde{u}(x, T)\|_{L_1(\Sigma)} < \delta_1 |\Sigma|
\end{equation}
\begin{equation}
      \| W_{{F}_{\boldsymbol\theta}}(x, T)\|_{L_1(\D)} < \delta_2 |\D|
\end{equation}
Also from condition 3 we have that $F_{\boldsymbol\theta}-\tilde{u}$ and $W_{F_{\boldsymbol\theta}}$ are both $l$-Lipschitz continuous on $\bar{\D}$. 
From Lemma~\ref{lem:Lipshitz_bound} we know that
\begin{equation}
\begin{aligned}
    & \quad \sup _{(x, T) \in \Sigma}| {F}_{\boldsymbol\theta}(x, T)-\Tilde{u}(x, T)|  \\
    & \leq \max \biggl\{\frac{2\| {F}_{\boldsymbol\theta}(x, T)-\Tilde{u}(x, T)\|_{L_1(\Sigma)}}{R_{\Sigma}|\Sigma|}, \\
    & \qquad 2 l \cdot\left(\frac{\| {F}_{\boldsymbol\theta}(x, T)-\Tilde{u}(x, T)\|_{L_1(\Sigma)} \cdot \Gamma(d / 2+1)}{l R_{\Sigma} \cdot \pi^{d / 2}}\right)^{\frac{1}{d+1}}\biggl\} \\
    & < \max \left\{\frac{2 \delta_1 |\Sigma| }{R_{\Sigma}|\Sigma|}, 2 l \cdot\left(\frac{\delta_1 |\Sigma| \cdot \Gamma(d / 2+1)}{l R_{\Sigma} \cdot \pi^{d / 2}}\right)^{\frac{1}{d+1}}\right\},
\end{aligned} 
\end{equation}
\begin{equation}
\begin{aligned}
    & \quad \sup _{(x, T) \in \D}|W_{{F}_{\boldsymbol\theta}}(x,T)| \\
    & \leq \max \biggl\{\frac{2\| W_{{F}_{\boldsymbol\theta}}(x, T)\|_{L_1(\D)}}{R_{\D}|\D|}, \\
    & \qquad 2 l \cdot\left(\frac{\| W_{{F}_{\boldsymbol\theta}}(x, T)\|_{L_1(\D)} \cdot \Gamma((d+1) / 2+1)}{l R_{\D} \cdot \pi^{(d+1) / 2}}\right)^{\frac{1}{d+2}}\biggl\} \\
    & < \max \left\{\frac{2\delta_2 |\D|}{R_{\D}|\D|}, 2 l \cdot\left(\frac{\delta_2 |\D| \cdot \Gamma((d+1) / 2+1)}{l R_{\D} \cdot \pi^{(d+1) / 2}}\right)^{\frac{1}{d+2}}\right\}.
\end{aligned}
\end{equation}
Let
\begin{equation}
\begin{aligned}
    \tilde \delta_1 & = \max \left\{\frac{2 \delta_1 |\Sigma| }{R_{\Sigma}|\Sigma|}, 2 l \cdot\left(\frac{\delta_1 |\Sigma| \cdot \Gamma(d / 2+1)}{l R_{\Sigma} \cdot \pi^{d / 2}}\right)^{\frac{1}{d+1}}\right\}, \\
    \tilde \delta_2 & = \max \left\{\frac{2\delta_2 |\D|}{R_{\D}|\D|}, 2 l \cdot\left(\frac{\delta_2 |\D| \cdot \Gamma((d+1) / 2+1)}{l R_{\D} \cdot \pi^{(d+1) / 2}}\right)^{\frac{1}{d+2}}\right\}.
\end{aligned}
\end{equation}
Then from Theorem~\ref{thm:PINN_bound} we know that 
\begin{equation}
    \sup _{(x,T) \in \D}\left|{F}_{\boldsymbol\theta}(x, T)-u(x, T)\right| \leq \tilde\delta_1 + C \frac{\tilde \delta_2}{\lambda}.
\end{equation}
Given that $\lambda = \frac{1}{2}\sigma^2$, replace $C$ with $2C$ we get
\begin{equation}
    \sup _{(x,T) \in \D}\left|{F}_{\boldsymbol\theta}(x, T)-u(x, T)\right| \leq \tilde\delta_1 + C \frac{\tilde \delta_2}{\sigma^2},
\end{equation}
which completes the proof.

\end{proof}

\newpage

\begin{IEEEbiography}[{\includegraphics[width=1in,height=1.25in,clip,keepaspectratio]{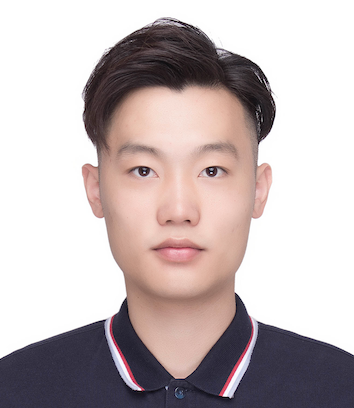}}]{Zhuoyuan Wang} received his B.E. degree in Automation from Tsinghua University, Beijing, China, in 2020 and is currently pursuing a Ph.D. degree in 
Electrical and Computer Engineering at Carnegie Mellon University, Pittsburgh, PA, USA.

His research interests include safety-critical control for stochastic systems, physics-informed learning, safe reinforcement learning and application to robotic systems.
He is a recipient of the Michel and Kathy Doreau Graduate Fellowship at Carnegie Mellon University.
\end{IEEEbiography}

\vspace{-20em}

\begin{IEEEbiography}[{\includegraphics[width=1in,height=1.25in,clip,keepaspectratio]{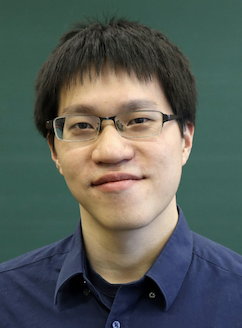}}]{Albert Chern} is an Assistant Professor in Computer Science and Engineering at University of California San Diego, La Jolla, CA, USA, since 2020.  He received his Ph.D. in Applied and Computational Mathematics at California Institute of Technology, Pasadena, CA, USA, in 2017, and was a Postdoctoral Researcher in Mathematics at Technische Universit{\"a}t Berlin, Berlin, Germany, from 2017 to 2020.  Chern's research interest lies in applications of differential geometry to computational math, fluid dynamics, computer graphics, and the interplay between stochastic processes and differential equations.  Chern was a recipient of the NSF CAREER Award in 2023.
\end{IEEEbiography}

\vspace{-20em}

\begin{IEEEbiography}[{\includegraphics[width=1in,height=1.25in,clip,keepaspectratio]{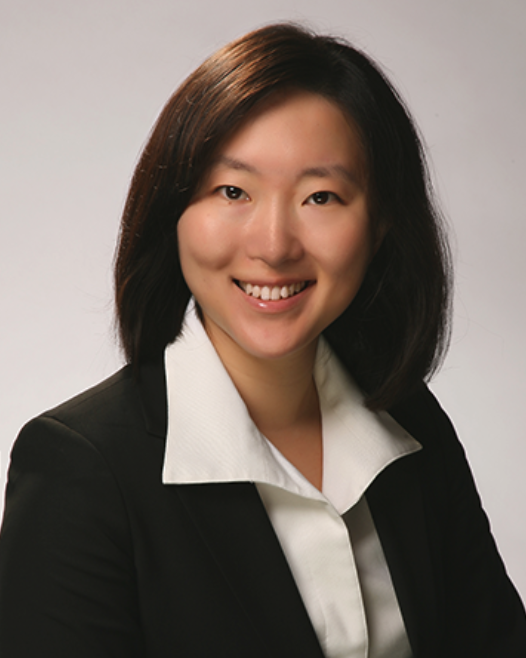}}]{Yorie Nakahira} is an Assistant Professor in the Department of Electrical and Computer Engineering at Carnegie Mellon University. She received B.E. in Control and Systems Engineering from Tokyo Institute of Technology in 2012 and Ph.D. in Control and Dynamical Systems from California Institute of Technology in 2019. Her research interests include the fundamental theory of optimization, control, and learning and its application to neuroscience, cell biology, smart grid, cloud computing, finance, autonomous robots.
\end{IEEEbiography}

\newpage

\onecolumn

\subsection{Proof of Theorems for Physics-informed Learning (Continued)}
\label{apx:proof_pinn_cont}

\begin{proof}
(Theorem~\ref{thm:pinn_convergence})

The proof is based on~\cite[Theorem 3.6]{shin2020convergence}. We first introduce some definitions.
A generic partial differential equation (PDE) has the form
\begin{equation}
    \mathcal{L}[u](x)=p(x), \; \forall x \in D, \quad \mathcal{B}[u](x)=q(x), \; \forall x \in \Gamma \subseteq \partial D,
\end{equation}
where $\mathcal{L}[\cdot]$ is a differential operator and $\mathcal{B}[\cdot]$ could be Dirichlet, Neumann, Robin, or periodic boundary conditions. 
We define the physics-model loss for the PINN training with data $m$ for a neural network instance $h$ as
\begin{equation}
\begin{aligned}
    \operatorname{Loss}_m^{\mathrm{PINN}}(h ; \boldsymbol{\lambda})= & \frac{\lambda_r}{m_r} \sum_{i=1}^{m_r}\left\|\mathcal{L}[h]\left(x_r^i\right)-a\left(x_r^i\right)\right\|^2+ \\
    & \quad \frac{\lambda_b}{m_b} \sum_{i=1}^{m_b}\left\|\mathcal{B}[h]\left(x_b^i\right)-b\left(x_b^i\right)\right\|^2,
\end{aligned}
\end{equation}
where $\boldsymbol{\lambda} = [\lambda_r, \lambda_b]$ are hyperparameters.

We now restate the assumptions used in~\cite{shin2020convergence}.

\begin{assumption}\cite[Assumption 3.1]{shin2020convergence}
\label{asp:shin_asp_3.1}

    Let $D \in \mathbb{R}^{d+1}$ and $\Gamma \in \mathbb{R}^{d}$ be the state-time space of interest and a subset of its boundary, respectively.
    Let $\mu_r$ and $\mu_b$ be probability distributions defined on $D$ and $\Gamma$.
    Let $\rho_r$ be the probability density of $\mu_r$ with respect to $(d+1)$-dimensional Lebesgue measure on $D$. Let $\rho_b$ be the probability density of $\mu_b$ with respect to $d$-dimensional Hausdorff measure on $\Gamma$.
    We assume the following conditions.
    \begin{enumerate}

    \item $D$ is at least of class $C^{0,1}$.

    \item  $\rho_r$ and $\rho_b$ are supported on $D$ and $\Gamma$, respectively. Also, $\inf _D \rho_r>0$ and $\inf _{\Gamma} \rho_b>0$.
    
    \item For $\epsilon>0$, there exists partitions of $D$ and $\Gamma,\left\{D_j^\epsilon\right\}_{j=1}^{K_r}$ and $\left\{\Gamma_j^\epsilon\right\}_{j=1}^{K_b}$ that depend on $\epsilon$ such that for each $j$, there are cubes $H_\epsilon\left(\mathbf{z}_{j, r}\right)$ and $H_\epsilon\left(\mathbf{z}_{j, b}\right)$ of side length $\epsilon$ centered at $\mathbf{z}_{j, r} \in D_j^\epsilon$ and $\mathbf{z}_{j, b} \in \Gamma_j^\epsilon$, respectively, satisfying $D_j^\epsilon \subset H_\epsilon\left(\mathbf{z}_{j, r}\right)$ and $\Gamma_j^\epsilon \subset H_\epsilon\left(\mathbf{z}_{j, b}\right)$.

    \item There exists positive constants $c_r, c_b$ such that $\forall \epsilon>0$, the partitions from the above satisfy $c_r \epsilon^d \leq \mu_r\left(D_j^\epsilon\right)$ and $c_b \epsilon^{d-1} \leq \mu_b\left(\Gamma_j^\epsilon\right)$ for all $j$.
    There exists positive constants $C_r, C_b$ such that $\forall x_r \in D$ and $\forall x_b \in \Gamma, \mu_r\left(B_\epsilon\left(x_r\right) \cap D\right) \leq C_r \epsilon^d$ and $\mu_b\left(B_\epsilon\left(x_b\right) \cap \Gamma\right) \leq C_b \epsilon^{d-1}$ where $B_\epsilon(x)$ is a closed ball of radius $\epsilon$ centered at $x$.
    Here $C_r, c_r$ depend only on $\left(D, \mu_r\right)$ and $C_b, c_b$ depend only on $\left(\Gamma, \mu_b\right)$.

    \item When $d=1$, we assume that all boundary points are available. Thus, no random sample is needed on the boundary.
    \end{enumerate}
\end{assumption}

\begin{assumption}\cite[Assumption 3.2]{shin2020convergence}
\label{asp:shin_asp_3.2}

    Let $k$ be the highest order of the derivative for the PDE of interest. For some $0<\alpha \leq 1$, let $p \in C^{0, \alpha}(D)$ and $q \in C^{0, \alpha}(\Gamma)$.
    \begin{enumerate}
        \item For each training dataset $m$, let $\mathcal{H}_m$ be a class of neural networks in $C^{k, \alpha}(D) \cap C^{0, \alpha}(\bar{D})$ such that for any $h \in \mathcal{H}_m, \mathcal{L}[h] \in C^{0, \alpha}(D)$ and $\mathcal{B}[h] \in C^{0, \alpha}(\Gamma)$.

        \item 
        For each $m, \mathcal{H}_m$ contains a network $u_m^*$ satisfying $\operatorname{Loss}_m^{\mathrm{PINN}}\left(u_m^* ; \lambda\right)=0$.

        \item 
        And,
        $$
        \sup _m\left[\mathcal{L}\left[u_m^*\right]\right]_{\alpha ; D}<\infty, \quad \sup _m\left[\mathcal{B}\left[u_m^*\right]\right]_{\alpha ; \Gamma}<\infty .
        $$  
    \end{enumerate}
\end{assumption}

We also know that safety-related PDEs are second-order linear parabolic equations. Let $U$ be a bounded domain in $\mathbb{R}^d$ and let $U_T=U \times(0, T]$ for some fixed time $T>0$. 
Let $(x, t)$ be a point in $\mathbb{R}^{d+1}$.
For a parabolic PDE, we define
\begin{equation}
    \begin{cases}-u_t+L[u]=s, & \text { in } U_T \\ u=\varphi, & \text { in } \partial U \times[0, T] \\ u=v, & \text { in } \bar{U} \times\{t=0\}\end{cases}
\end{equation}
where $s: U_T \mapsto \mathbb{R}, v: \bar{U} \mapsto \mathbb{R}, \varphi: \partial U \times[0, T] \mapsto \mathbb{R}$, and
\begin{equation}
\begin{aligned}
    L[u]= & \sum_{i, j=1}^d D_i\left(a^{i j}(x, t) D_j u+b^i(x, t) u\right) \\
    & \quad +\sum_{i=1}^d c^i(x, t) D_i u+d(x, t) u .
\end{aligned}
\end{equation}

\begin{assumption}\cite[Assumption 3.4]{shin2020convergence}
\label{asp:shin_asp_3.4}

Let $\lambda(x, t)$ be the minimum eigenvalues of $\left[a^{i j}(x, t)\right]$ and $\alpha \in(0,1)$. Suppose $a^{i j}, b^i$ are differentiable and let $\tilde{b}^i(x, t)=\sum_{j=1}^d D_j a^{i j}(x, t)+b^i(x, t)+c^i(x, t)$ and $\tilde{c}(x, t)=$ $\sum_{i=1}^d D_i b^i(x, t)+d(x, t)$. Let $\Omega=U \times(0, T)$.
\begin{enumerate}
    \item For some constant $\lambda_0, \lambda(x, t) \geq \lambda_0>0$ for all $(x, t) \in \Omega$.

    \item 
    $a^{i j}, \tilde{b}^i, \tilde{c}^i$ are Hölder continuous (exponent $\alpha$) in $\Omega$, and $\left|a^{i j}\right|_{\alpha,}\left|\mathrm{d} \tilde{b}^i\right|_{\alpha,}\left|\mathrm{d}^2 \tilde{c}\right|_\alpha$ are uniformly bounded.

    \item 
    $s$ is Hölder continuous (exponent $\alpha$) in $U_T$ and $\left|\mathrm{d}^2 f\right|_\alpha<\infty$. $v$ and $\varphi$ are Hölder continuous (exponent $\alpha$) on $\bar{U} \times\{t=0\}$, and $\partial U \times[0, T]$, respectively, and $v(x)=\varphi(x, 0)$ on $\partial U$.

    \item 
    There exists $\theta>0$ such that $\theta^2 a^{11}(x, t)+\theta \tilde{b}^1(x, t) \geq 1$ in $\Omega$.

    \item 
    For $x^{\prime} \in \partial U$, there exists a closed ball $\bar{B}$ in $\mathbb{R}^d$ such that $\bar{B} \cap \bar{U}=\left\{x^{\prime}\right\}$.

    \item 
    There are constants $\Lambda, v \geq 0$ such that for all $(x, t) \in \Omega$
\begin{equation}
\begin{aligned}
    \sum_{i, j}\left|a^{i j}(x, t)\right|^2 & \leq \Lambda, \\
    \lambda_0^{-2}\left(\|\boldsymbol{b}(x, t)\|^2+\|\boldsymbol{c}(x, t)\|^2\right) & +\lambda_0^{-1}|d(x, t)| \leq v^2,
\end{aligned}
\end{equation}
where $\boldsymbol{b}(x, t)=\left[b^1(x, t), \cdots, b^d(x, t)\right]^T$ and $\boldsymbol{c}(x, t)=\left[c^1(x, t), \cdots, c^d(x, t)\right]^T$.
\end{enumerate}
\end{assumption}

Now that we have listed all assumptions needed in~\cite[Theorem 3.6]{shin2020convergence}, we are ready to prove Theorem~\ref{thm:pinn_convergence}. 

Note that conditions 1-5 in Assumption~\ref{asp:domain} match with Assumption~\ref{asp:shin_asp_3.1} (\cite[Assumption 3.1]{shin2020convergence}), thus it is satisfied. 

Since we assume for any training data set $m_r$, architecture-wise the physics-informed neural network has enough representation power to characterize the solution of the sasfety-related PDE $u$, Assumption~\ref{asp:shin_asp_3.2} (\cite[Assumption 3.2]{shin2020convergence}) is satisfied.

For Assumption~\ref{asp:shin_asp_3.4} (\cite[Assumption 3.4]{shin2020convergence}), condition 1 is satisfied because we assume the noise magnitude $\sigma$ has non-zero eigenvalues in Assumption~\ref{asp:dynamics}, and $a(x,t) = \sigma^\intercal \sigma$ which indicates all eigenvalues of $a$ are greater than 0, thus $\lambda_0$ can be found and this condition is satisfied.
For condition 2, we have $f$, $g$ and $\sigma$ in the system dynamics are bounded and continuous from Assumption~\ref{asp:dynamics}, and we know $a = \sigma^\intercal \sigma$, $b \equiv 0$, $c = f + g \, u_{\text{nominal}}$ and $d \equiv 0$. Thus, $a^{i j}, \tilde{b}^i, \tilde{c}^i$ are Hölder continuous (exponent $\alpha$) in $\Omega$, and $\left|a^{i j}\right|_{\alpha,}\left|\mathrm{d} \tilde{b}^i\right|_{\alpha,}\left|\mathrm{d}^2 \tilde{c}\right|_\alpha$ are uniformly bounded, indicating the condition is satisfied. 
For condition 3, we know that $s \equiv 0$ and $\varphi \equiv 0$ or $\varphi \equiv 1$, and $v \equiv 0$ or $v \equiv 1$ depending on the specific choice of the safety-related probability, thus they are Hölder continuous and we know that condition 3 is satisfied. 
For condition 4, since we have $a^{11} > 0$ in Assumption~\ref{asp:dynamics} and we know $b \equiv 0$, we can always find $\theta > 0$ so that the condition is satisfied.
For condition 5, it is assumed in the regularity condition in Assumption~\ref{asp:domain} thus is satisfied.
For condition 6, we have $f$, $g$, $\sigma$ bounded in Assumption~\ref{asp:dynamics} and we know that $a = \sigma^\intercal \sigma$, $b \equiv 0$, $c = f + g \, u_{\text{nominal}}$ and $d \equiv 0$. Thus, we can find constants $\Lambda, v \geq 0$ to bound $\sum_{i, j}\left|a^{i j}(x, t)\right|^2$ and $\lambda_0^{-2}\left(\|\boldsymbol{b}(x, t)\|^2+\|\boldsymbol{c}(x, t)\|^2\right) +\lambda_0^{-1}|d(x, t)|$ for all $(x, t) \in \Omega$, indicating this condition is satisfied.

Till here, we have shown that all assumptions needed in~\cite[Theorem 3.6]{shin2020convergence} are satisfied.
As the number of training data $m_r \rightarrow \infty$, the loss function~\eqref{eq:PINN overall loss function} is equivalent to the H\"{o}lder regularized empirical loss~\cite[Equation (3.3)]{shin2020convergence}. Thus, all conditions in~\cite[Theorem 3.6]{shin2020convergence} hold and we have
\begin{equation*}
        \lim _{m_r \rightarrow \infty} F_{m_r}=u.
\end{equation*}

\end{proof}

\section{Experiment Details}
In the following sections, we provide details for the experiments presented in the paper and some additional experiments.

\subsection{Generalization to unseen regions}
In the generalization task in section~\ref{sec:generalization}, we use a down-sampled sub-region of the system to train the proposed PIPE framework, and test the prediction results on the full state-time space. We showed that PIPE is able to give accurate inference on the entire state space, while standard fitting strategies cannot make accurate predictions. In the paper, we only show the fitting results of thin plate spline interpolation. Here, we show the results of all fitting strategies we tested for this generalization tasks. The fitting strategies are
\begin{enumerate}
    \item Polynomial fitting of 5 degrees for both state $x$ and time $T$ axes. The fitting sum of squares error (SSE) on the training data is $0.1803$.
    \item Lowess: locally weighted scatterplot smoothing~\cite{cleveland1981lowess}. The training SSE is $0.0205$.
    \item Cubic spline interpolation. The training SSE is $0$.
    \item Biharmonic spline interpolation. The training SSE is $2.52\times 10^{-27}$.
    \item TPS: thin plate spline interpolation. The training SSE is $1.64 \times 10^{-26}$.
\end{enumerate}
All fittings are conducted via the MATLAB Curve Fitting Toolbox.
Fig.~\ref{fig:fitting} visualizes the fitting results on the full state space. Polynomial fitting performs undesirably because the polynomial functions cannot capture the risk probability geometry well. Lowess fitting also fails at inference since it does not have any model information of the data. Given the risk probability data, cubic spline cannot extrapolate outside the training region, and we use $0$ value to fill in the unseen region where it yields NAN for prediction. Biharmonic and TPS give similar results as they are both spline interpolation methods. None of these fitting methods can accurately predict the risk probability in unseen regions, because they purely rely on training data and do not incorporate any model information of the risk probability for prediction.

\begin{figure}[h]
    \centering
    \includegraphics[width=11cm]{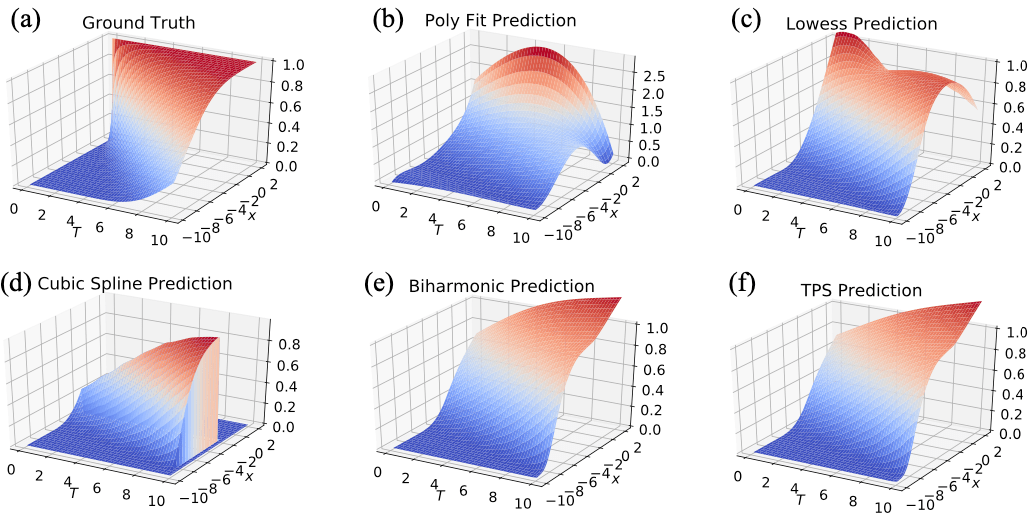}
    \caption{Results of different fitting strategies on the risk probability generalization task.}
    \label{fig:fitting}
\end{figure}

We also compare the prediction results for different network architectures in the proposed PIPE framework, to examine the effect of network architectures on the risk probability prediction performance.
The network settings we consider are different hidden layer numbers (1-4) and different numbers of neurons in one hidden layer (16, 32, 64). We use 3 hidden layers, 32 neurons per layer as baseline (the one used in the paper). Table~\ref{tb:prediction layer number} and Table~\ref{tb:prediction neuron number} report the averaged absolute error of the predictions for different layer numbers and neuron numbers per layer, respectively. We trained the neural networks 10 times with random initialization to report the standard deviation.
We can see that as the number of layers increases, the prediction error of the risk probability drops, but in a relatively graceful manner. The prediction error for a single layer PIPE is already very small, which indicates that the proposed PIPE framework is very efficient in terms of computation and storage. The prediction accuracy tends to saturate when the hidden layer number reaches 3, as there is no obvious improvement when we increase the layer number from 3 to 4. This means for the specific task we consider, a 3-layer neural net has enough representation.
Under the same layer number, as the neuron number per layer increases, the risk probability prediction error decreases. This indicates that with larger number of neuron in each layer (\ie wider neural networks), the neural network can better capture the mapping between state-time pair and the risk probability. However, the training time increases significantly for PIPEs with more neurons per layer ($152\mathrm{s}$ for 16 neurons and $971\mathrm{s}$ for 64 neurons), and the gain in prediction accuracy becomes marginal compared to the amount of additional computation involved. We suggest to use a moderate number of neurons per layer to achieve desirable trade-offs between computation and accuracy.

\begin{table*}[h]
\centering
\begin{tabular}{l|cccc}
    \hline
    \textbf{\# Hidden Layer} & 1 & 2 & 3 & 4 \\
    \hline
    \textbf{Prediction Error ($\times 10^{-3}$)} & $4.773\pm 0.564$ & $\boldsymbol{2.717 \pm 0.241}$ & $2.819 \pm 0.619$ & $2.778 \pm 0.523$ \\
    \hline
\end{tabular}
\caption{Risk probability prediction error of PIPE for different numbers of hidden layers.}
\label{tb:prediction layer number}
\end{table*}

\begin{table*}[h]
\centering
\begin{tabular}{l|ccc}
    \hline
    \textbf{\# Neurons} & 16 & 32 & 64 \\
    \hline
    \textbf{Prediction Error ($\times 10^{-3}$)} & $2.743 \pm 0.313$ & $2.931 \pm 0.865$ & $\boldsymbol{2.599 \pm 0.351}$ \\
    \hline
\end{tabular}
\caption{Risk probability prediction error of PIPE for different neuron numbers per layer.}
\label{tb:prediction neuron number}
\end{table*}

\subsection{Efficient estimation of risk probability}
In the efficient estimation task in section~\ref{sec:estimation}, we showed that PIPE will give better sample efficiency in risk probability prediction, in the sense that it achieves the same prediction accuracy with less sample numbers. Here, we visualize the prediction errors of Monte Carlo (MC) and the proposed PIPE framework to better show the results. Fig.~\ref{fig:comparison} shows the prediction error comparison plots for MC and PIPE with different sample numbers $N$. As the sample number increases, the errors for both MC and PIPE decrease because of better resolution of the data. PIPE gives 
more accurate predictions than MC across all sample numbers, since it combines data and physical model of the system together. From Fig.~\ref{fig:comparison} we can see that, PIPE indeed provides smoother and more accurate estimation of the risk probability. The visualization results further validate the efficacy of the proposed PIPE framework.
\begin{figure}
    \centering
    \includegraphics[width=15cm]{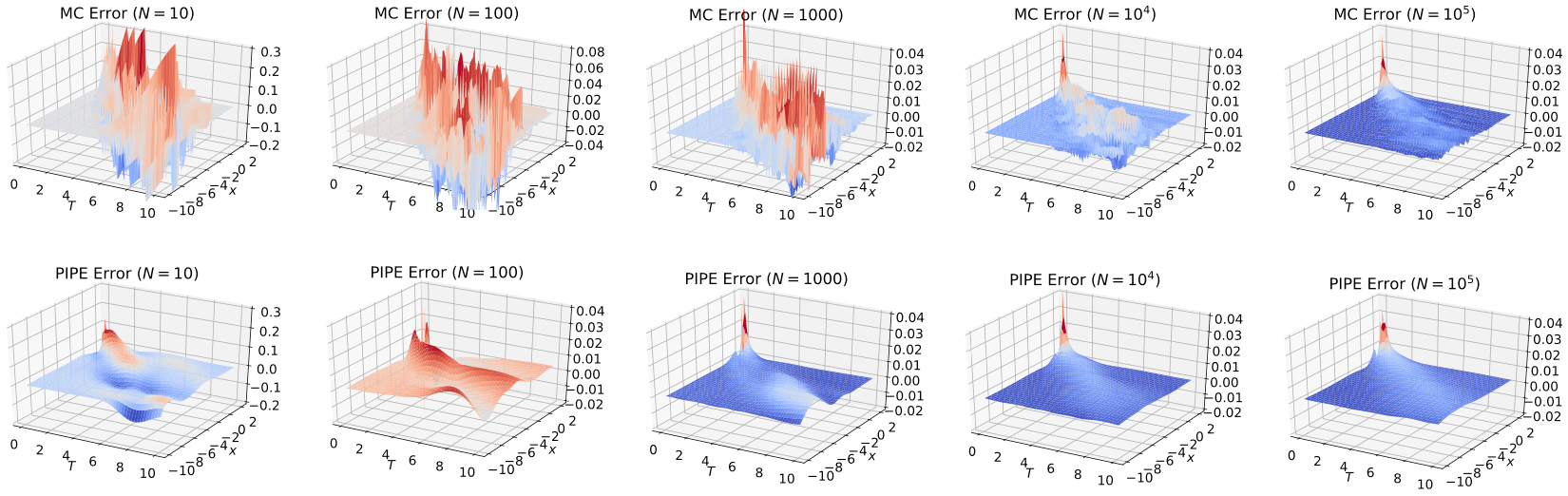}
    \caption{Prediction errors for Monte Carlo (left) and PIPE (right) with different sample numbers.}
    \label{fig:comparison}
\end{figure}

\begin{table*}[h]
\centering
\begin{tabular}{l|cccc}
    \hline
    \textbf{\# Hidden Layer} & 1 & 2 & 3 & 4 \\
    \hline
    \textbf{Prediction Error ($\times 10^{-4}$)} & $14.594 \pm 2.109$ & $7.302 \pm 0.819$ & $6.890 \pm 0.613$ & $\boldsymbol{6.625 \pm 0.574}$ \\
    \hline
\end{tabular}
\caption{Risk probability gradient prediction error of PIPE for different numbers of hidden layers.}
\label{tb:gradient layer number}
\end{table*}

\begin{table*}
\centering
\begin{tabular}{l|ccc}
    \hline
    \textbf{\# Neurons} & 16 & 32 & 64 \\
    \hline
    \textbf{Prediction Error ($\times 10^{-4}$)} & $7.049 \pm 0.767$ & $6.890 \pm 0.613$ & $\boldsymbol{6.458 \pm 0.794}$ \\
    \hline
\end{tabular}
\caption{Risk probability gradient prediction error of PIPE for different neuron numbers per layer.}
\label{tb:gradient neuron number}
\end{table*}

\subsection{Adaptation on changing system parameters}
For the adaptation task described in section~\ref{sec:adaptation}, we trained PIPE with system data of parameters $\lambda_{\text{train}} = [0.1,0.5,0.8,1]$ and tested over a range of unseen parameters over the interval $\lambda=[0,2]$. Here, we show additional results on parameters $\lambda_{\text{test}} = [0.3, 0.7, 1.2, 2]$ to further illustrate the adaptation ability of PIPE. Fig.~\ref{fig:varying-new} shows the results. It can be seen that PIPE is able to predict the risk probability accurately on both system parameters with very low error over the entire state-time space. This result indicates that PIPE has solid adaptation ability on uncertain parameters, and can be used for stochastic safe control with adaptation requirements.

\begin{figure}
    \centering
    \includegraphics[width=12cm]{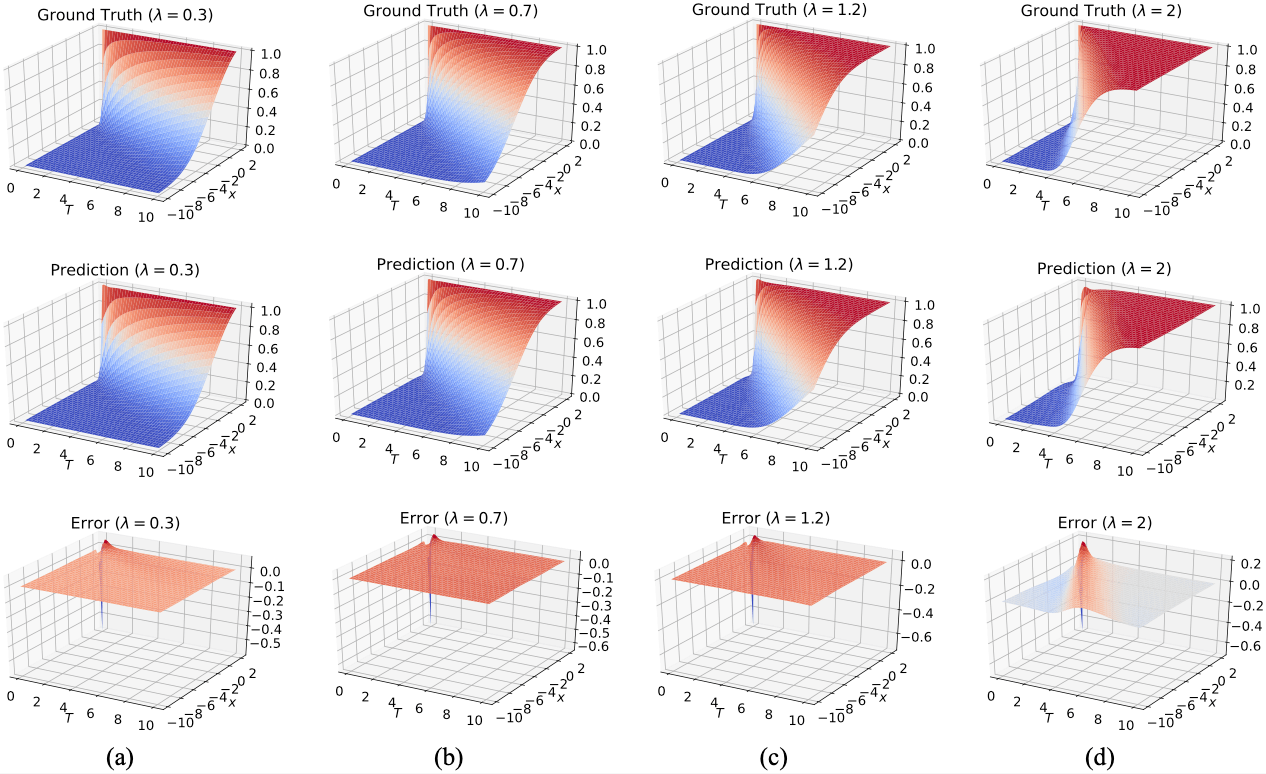}
    \caption{Risk probability prediction results on systems with unseen parameters $\lambda_{\text{test}} = [0.3, 0.7, 1.2, 2]$.}
    \label{fig:varying-new}
\end{figure}

\subsection{Estimating the gradient of risk probability}
For the gradient estimation task in section~\ref{sec:gradient}, we presented that PIPE is able to predict the risk probability gradients accurately by taking finite difference on the predicted risk probabilities. This result shows that PIPE can be used for first- and higher-order methods for safe control, by providing accurate gradient estimations in a real-time fashion. Similar to the generalization task, here we report the gradient prediction errors with different network architectures in PIPE, to examine the effect of network architectures on the gradient estimation performance. Table~\ref{tb:gradient layer number} and Table~\ref{tb:gradient neuron number} show the averaged absolute error of gradient predictions for different layer numbers and neuron numbers per layer. We trained the neural networks for 10 times with random initialization to report the standard deviation. 
We can see that as the number of hidden layer increases, the gradient prediction error keeps dropping, and tends to saturate after 3 layers.  
With the increasing neuron numbers per layer, the gradient prediction error decreases in a graceful manner. Similar to the generalization task, larger networks with more hidden layers and more neurons per layer can give more accurate estimation of the gradient, but the computation time scales poorly compared to the accuracy gain. Based on these results, we suggest to use moderate numbers of layers and neurons per layer to acquire desirable gradient prediction with less computation time.

\section{Additional Experiment Results}
In this section, we provide additional experiment results on risk estimation of the inverted pendulum on a cart system, as well as safe control with risk estimation through the proposed PIPE framework.

\subsection{Risk estimation of inverted pendulum on a cart system}
We consider the inverted pendulum on a cart system, with dynamics given by
\begin{equation}
\label{eq:pendulum_dynamics}
    \frac{d\mathbf{x}}{dt} = f(\mathbf{x}) + g(\mathbf{x}) u + \sigma \ \Tilde{I} \ dW_t,
\end{equation}
where $\mathbf{x} = [x, \dot x, \theta, \dot \theta]^\top$ is the state of the system and $u \in \mathbb{R}$ is the control, with $x$ and $\dot x$ being the position and velocity of the cart, and $\theta$ and $\dot \theta$ being the angle and angular velocity of the pendulum. We use $\mathbf{x}$ to denote the state of the system to distinguish from cart's position $x$. Then, we have
\begin{equation}
f(\mathbf{x}) = \left[\begin{array}{cccc}
1 & 0 & 0 & 0 \\
0 & m+M & 0 & m l \cos \theta \\
0 & 0 & 1 & 0 \\
0 & m l \cos \theta & 0 & m l^2 
\end{array}\right]^{-1} \left[\begin{array}{c}
\dot{x} \\
m l \dot{\theta}^2 \sin \theta-b_x \dot{x} \\
\dot{\theta} \\
m g l \sin \theta-b_\theta l \dot{\theta}
\end{array}\right],
\end{equation}
\begin{equation}
g(\mathbf{x}) = \left[\begin{array}{cccc}
1 & 0 & 0 & 0 \\
0 & m+M & 0 & m l \cos \theta \\
0 & 0 & 1 & 0 \\
0 & m l \cos \theta & 0 & m l^2 
\end{array}\right]^{-1} \left[\begin{array}{l}
0 \\
1 \\
0 \\
0
\end{array}\right],
\end{equation}
where $m$ and $M$ are the mass of the pendulum and the cart, $g$ is the acceleration of gravity, $l$ is the length of the pendulum, and $b_x$ and $b_\theta$ are constants. The last term in~\eqref{eq:pendulum_dynamics} is the additive noise, where $W_t$ is $4$-dimensional Wiener process with $W_0 = \mathbf{0}$, $\sigma$ is the magnitude of the noise, and
\begin{equation}
\Tilde{I} = \left[\begin{array}{cccc}
0 & 0 & 0 & 0 \\
0 & 1 & 0 & 0 \\
0 & 0 & 0 & 0 \\
0 & 0 & 0 & 1
\end{array}\right].
\end{equation}
Fig.~\ref{fig:cartpend} visualizes the system.

The safe set is defined in~\eqref{eq:safe_region} with barrier function $\phi(\mathbf{x}) = 1-(\frac{\mathbf{x}_3}{\pi/3})^2 = 1-(\frac{\theta}{\pi/3})^2$. Essentially the system is safe when the angle of the pendulum is within $[-\pi/3, \pi/3]$.
\begin{figure}
\includegraphics[width=6cm]{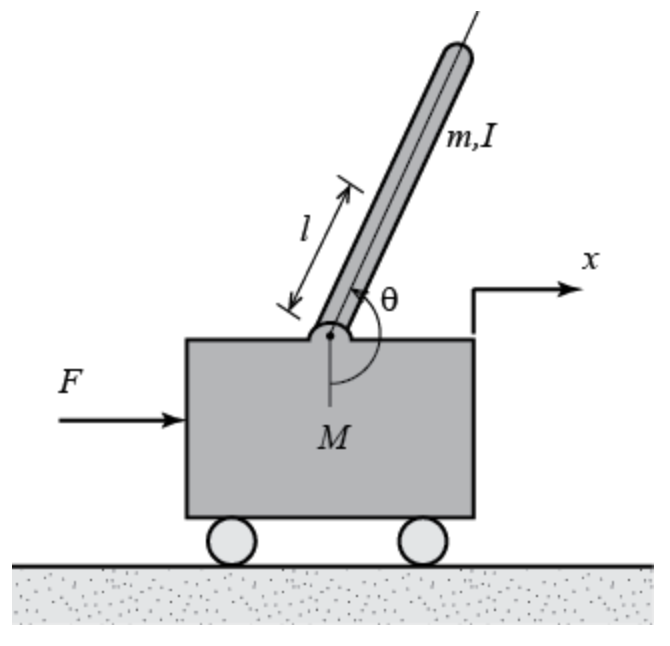}
\centering
\caption{Inverted pendulum on a cart.}
\label{fig:cartpend}
\end{figure}
We consider spatial-temporal space $\Omega \times \tau = \left[[-10,10] \times [-\pi/3, \pi/3] \times [-\pi, \pi] \right] \times [0,1]$. We collect training and testing data on $\Omega \times \tau$ with grid size $N_{\Omega\text{-train}}=13$ and $N_{\tau\text{-train}} = 10$ for training and $N_{\Omega\text{-test}}=25$ and $N_{\tau\text{-test}} = 10$ for testing. The nominal controller we choose is a proportional controller $N(\mathbf{x}) = -K \mathbf{x}$ with $K = [0, -0.9148, -22.1636, -14.3992]^\top$.
The sample number for MC simulation is set to be $N = 1000$ for both training and testing.
Table~\ref{tb:pendulum_parameters} lists the parameters used in the simulation.

We train PIPE with the same configuration listed in section~\ref{sec:experiments}. 
According to Theorem~\ref{thm:InvariantProbability_MainTheorem1}, the PDE that characterizes the safety probability of the pendulum system is
\begin{equation}
\label{eq:pendulum_pde}
\begin{aligned}
    \frac{\partial F}{\partial T}(\mathbf{x}, T) 
    & = \left(f(\mathbf{x}) - g(\mathbf{x}) K \mathbf{x} \right) \frac{\partial F}{\partial \mathbf{x}}(\mathbf{x}, T) + \frac{1}{2} \sigma^2 \Tilde{I} \operatorname{tr}\left( \frac{\partial^{2} F}{\partial \mathbf{x}^{2}}(\mathbf{x}, T) \right),
\end{aligned}
\end{equation}
which is a high dimensional and highly nonlinear PDE that cannot be solved effectively using standard solvers. Here we can see the advantage of combining data with model and using a learning-based framework to estimate the safety probability.
Fig.~\ref{fig:zoomin}, Fig.~\ref{fig:pendulum_t03} and Fig.~\ref{fig:pendulum_t1} show the results of the PIPE predictions. We see that despite the rather high dimension and nonlinear dynamics of the pendulum system, PIPE is able to predict the safety probability of the system with high accuracy. Besides, since PIPE takes the model knowledge into training loss, the resulting safety probability prediction is smoother thus more reliable than pure MC estimations.

\begin{table*}
\centering
\begin{tabular}{|c|c|}
    \hline
    Parameters & Values \\
    \hline
    $M$ & $1$ \\
    $m$ & $0.1$ \\
    $g$ & $9.8$ \\
    $l$ & $0.5$ \\
    $b_x$ & $0.05$ \\
    $b_\theta$ & $0.1$ \\
    \hline
\end{tabular}
\caption{Parameters used in the inverted pendulum simulation.}
\label{tb:pendulum_parameters}
\end{table*}

\begin{figure}
\includegraphics[width=15cm]{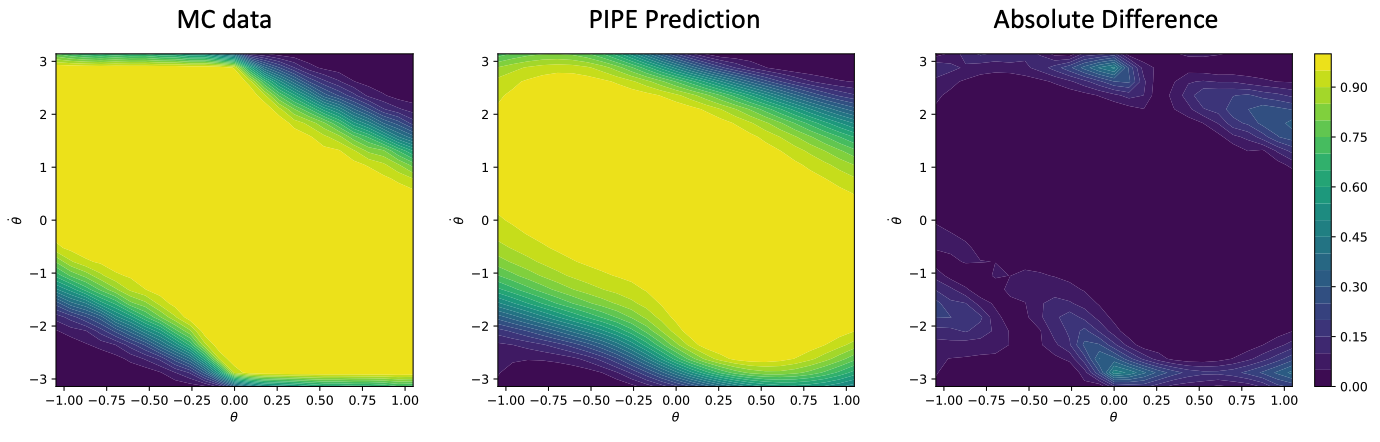}
\centering
\caption{Safety probability from MC simulation and PIPE prediction, and their absolute difference. Results on outlook time horizon $T = 0.6$, initial velocity $v = 0$. The x-axis shows the initial angle, and y-axis shows the initial angular velocity.}
\label{fig:zoomin}
\end{figure}

\begin{figure}
\includegraphics[width=16cm]{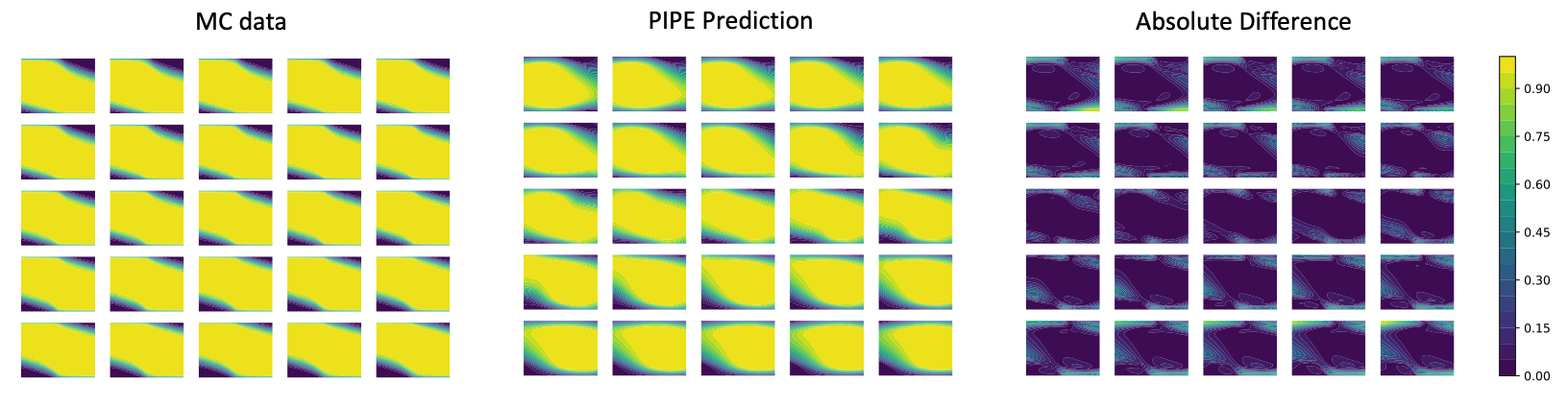}
\centering
\caption{Safety probability from MC simulation and PIPE prediction, and their absolute difference. Results on outlook time horizon $T = 0.3$. The 5x5 plots show results on 25 different initial velocities uniformly sampled in $[-10, 10]$. The x-axis and y-axis (omitted) are the initial angle and the initial angular velocity as in Fig.~\ref{fig:zoomin}. One can see that the safety probability shift as the velocity changes, and the safety probability is symmetric with regard to the origin $v=0$ due to the symmetry of the dynamics.}
\label{fig:pendulum_t03}
\end{figure}

\begin{figure}
\includegraphics[width=16cm]{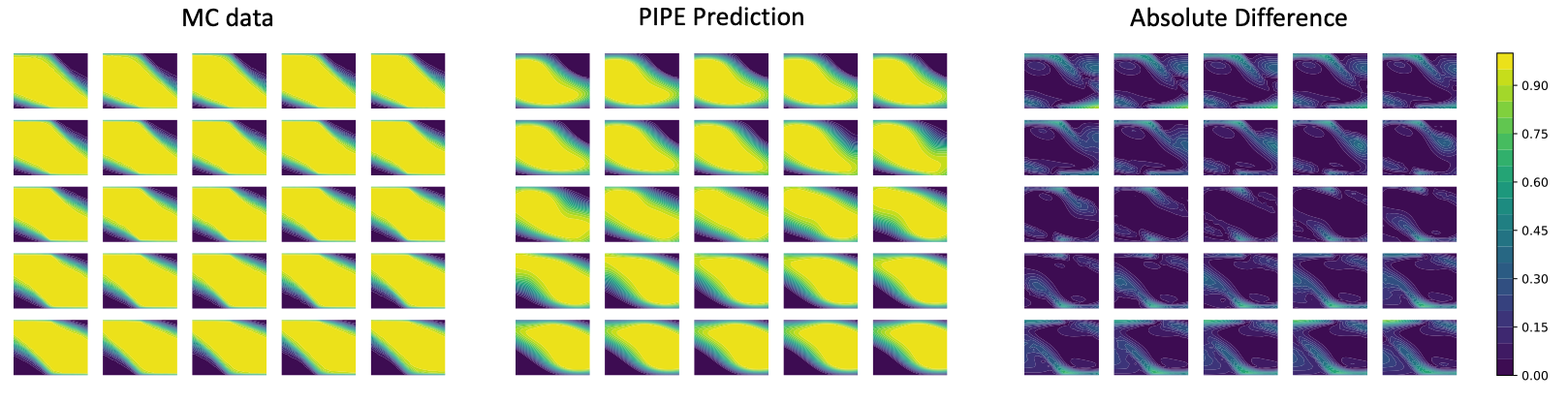}
\centering
\caption{Safety probability from MC simulation and PIPE prediction, and their absolute difference. Results on outlook time horizon $T = 1$. The 5x5 plots show results on 25 different initial velocities uniformly sampled in $[-10, 10]$.}
\label{fig:pendulum_t1}
\end{figure}

\subsection{Safe control with PIPE}
We consider the control affine system~\eqref{eq:x_trajectory} with $f(x) \equiv Ax = 2x$, $g(x) \equiv 1$, $\sigma(x) \equiv 2$. 
The safe set is defined as in~\eqref{eq:safe_region} and the barrier function is chosen to be $\phi(x) := x-1$. The safety specification is given as the forward invariance condition. 
The nominal controller is a proportional controller $N(x) = -K x$ with $K = 2.5$. The closed-loop system with this controller has an equilibrium at $x=0$ and tends to move into the unsafe set in the state space. We run simulations with $d t = 0.1$ for all controllers. The initial state is set to be $x_0 = 3$. 
For this system, the safety probability satisfies the following PDE
\begin{equation}
\label{eq:safety pde}
\begin{aligned}
    \frac{\partial F}{\partial T}(x, T) 
    & = -0.5x \frac{\partial F}{\partial x}(x, T) + 2 \operatorname{tr}\left( \frac{\partial^{2} F}{\partial x^{2}}(x, T) \right),
\end{aligned}
\end{equation}
with initial and boundary conditions
\begin{equation}
\begin{aligned}
\label{eq:icbc}
F(x,t) & = 0, \; x \leq 1, \\
F(x,0) & = \mathbbm{1}(x \geq 1).
\end{aligned}
\end{equation}

We first estimate the safety probability $F(x,T)$ of the system via PIPE. 
The training data $\bar{F}(x,T)$ is acquired by running MC on the system dynamics for given initial state $x_0$ and nominal control $N$. Specifically,
\begin{equation}
    \bar{F}(x,T) = \pr(\forall t \in [0,T], x_t \in \mathcal{C} \mid x_0 = x) = \frac{N_{\text{safe}}}{N},
\end{equation}
with $N=100$ being the number of sample trajectories. The training data is sampled on the state-time region $\Omega \times \tau = [1,10] \times [0,10]$ with grid size $dx = 0.5$ and $dt = 0.5$. 
We train PIPE with the same configuration as listed in section~\ref{sec:experiments}.
We test the estimated safety probability and its gradient on the full state space $\Omega \times \tau$ with $dx = 0.1$ and $dt = 0.1$. Fig.~\ref{fig:PIPE estimation} shows the results. It can be seen that the PIPE estimate is very close to the Monte Carlo samples, which validates the efficacy of the framework. Furthermore, the PIPE estimation has smoother gradients, due to the fact that it leverages model information along with the data.
\begin{figure}
    \centering
    \includegraphics[width=11cm]{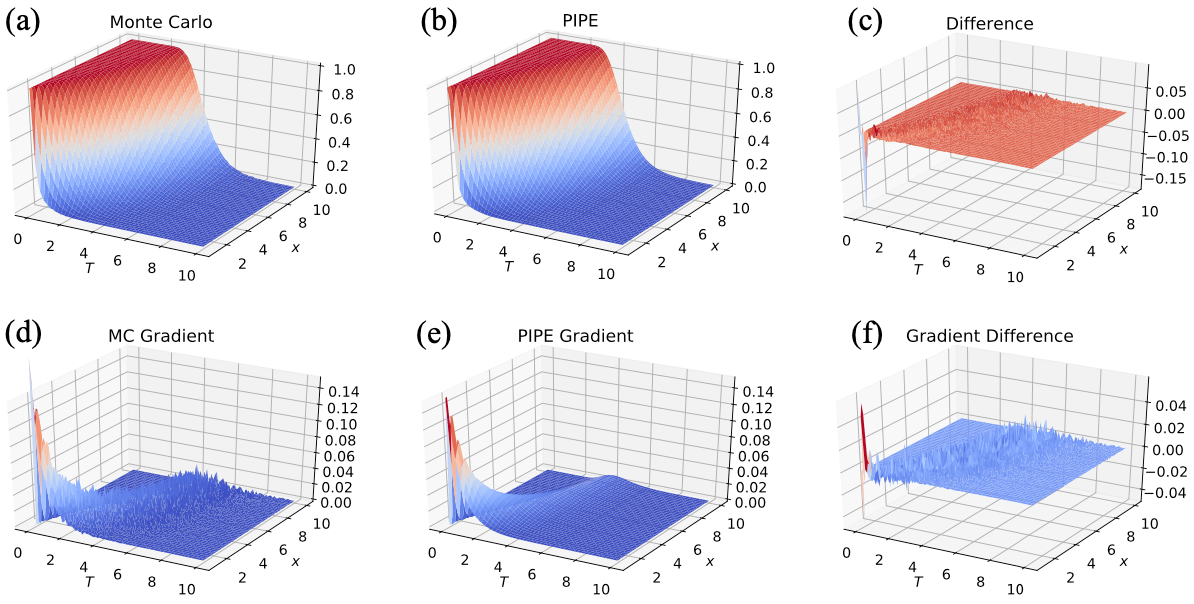}
    \caption{Safety probability and its gradient of Monte Carlo samples and PIPE estimation.}
    \label{fig:PIPE estimation}
\end{figure}

We then show the results of using such estimated safety probability for control. Fig.~\ref{fig:PIPE control} shows the results. For the baseline stochastic safe controllers for comparison, refer to~\cite{wang2022myopically} for details. We see that the PIPE framework can be applied to long-term safe control methods discussed in~\cite{wang2022myopically}. Together with the PIPE estimation, long-term safety can be ensured in stochastic control systems.
\begin{figure}
    \centering
    \includegraphics[width=12cm]{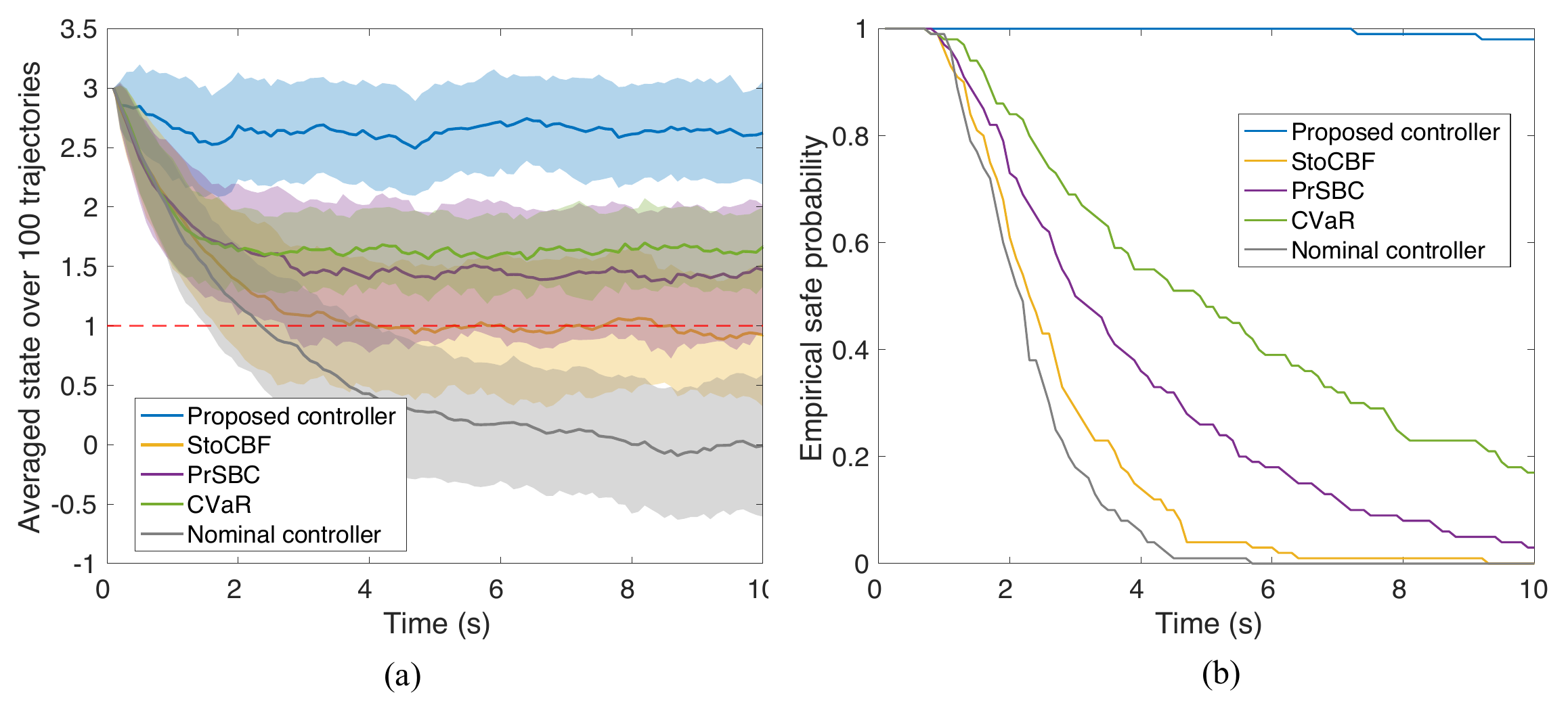}
    \caption{Safe control with probability estimation from PIPE compared with other baselines.}
    \label{fig:PIPE control}
\end{figure}

\end{document}